%% file: main_arvix.tex
\documentclass{article}

\usepackage{microtype}
\usepackage{graphicx}
\usepackage{subcaption}
\usepackage{booktabs} 
\usepackage{tabularx}             
\usepackage{multirow}             
\usepackage{enumitem}          
\usepackage{hyperref}
\usepackage{stfloats} 
\usepackage{afterpage}

\usepackage{pifont}               
\newcommand{\cmark}{\ding{51}}
\newcommand{\xmark}{\ding{55}}

\usepackage[accepted]{arxiv}

\usepackage{amsmath}
\usepackage{amssymb}
\usepackage{mathtools}
\usepackage{amsthm}

\usepackage[capitalize,noabbrev]{cleveref}

\theoremstyle{plain}
\newtheorem{theorem}{Theorem}[section]

\theoremstyle{definition}

\theoremstyle{remark}

\usepackage[textsize=tiny]{todonotes}

\include{custom_notation}

\begin{document}

\twocolumn[
  \icmltitle{SymPlex: A Structure-Aware Transformer for Symbolic PDE Solving}




  \begin{icmlauthorlist}
    \icmlauthor{Yesom Park}{math}
    \icmlauthor{Annie C. Lu}{math}
    \icmlauthor{Shao-Ching Huang}{oarc}
    \icmlauthor{Qiyang Hu}{oarc}
    \icmlauthor{Y. Sungtaek Ju}{mech}
    \icmlauthor{Stanley Osher}{math}
  \end{icmlauthorlist}

  \icmlaffiliation{math}{Department of Mathematics, University of California, Los Angeles, Los Angeles, CA, USA}
  \icmlaffiliation{mech}{Mechanical and Aerospace Engineering, University of California, Los Angeles, Los Angeles, CA, USA }
  \icmlaffiliation{oarc}{Office of Advanced Research Computing (OARC),
University of California, Los Angeles, Los Angeles, CA, USA}

  \icmlcorrespondingauthor{Stanley Osher}{sjo@math.ucla.edu}


  \vskip 0.3in
]



\printAffiliationsAndNotice{}  

\begin{abstract}
We propose \textbf{SymPlex}, a reinforcement learning framework for discovering analytical symbolic solutions to partial differential equations (PDEs) without access to ground-truth expressions.
SymPlex formulates symbolic PDE solving as tree-structured decision-making and optimizes candidate solutions using only the PDE and its boundary conditions.
At its core is \textbf{SymFormer}, a structure-aware Transformer that models hierarchical symbolic dependencies via tree-relative self-attention and enforces syntactic validity through grammar-constrained autoregressive decoding, overcoming the limited expressivity of sequence-based generators.
Unlike numerical and neural approaches that approximate solutions in discretized or implicit function spaces, SymPlex operates directly in symbolic expression space, enabling interpretable and human-readable solutions that naturally represent non-smooth behavior and explicit parametric dependence.
Empirical results demonstrate exact recovery of non-smooth and parametric PDE solutions using deep learning–based symbolic methods.

\end{abstract}

\section{Introduction}
Analytical solutions to partial differential equations (PDEs) play a fundamental role in science and engineering, providing exact descriptions of physical phenomena and direct interpretability.
Unlike numerical approximations, closed-form expressions can represent non-smooth behavior exactly, generalize analytically beyond a computational domain, and expose parametric dependencies critical for inverse modeling, control, and bifurcation analysis.
Despite their importance, finding analytical PDE solutions automatically remains a largely unsolved challenge.

Most existing PDE solvers operate in approximate representation spaces.
Classical numerical methods, including finite difference (FDM)~\cite{richtmyer1959difference, leveque1998finite}, finite volume (FVM)~\cite{leveque2002finite, toro2013riemann}, and finite element (FEM) methods~\cite{ciarlet2002finite, hughes2003finite}, approximate solutions on discretized grids, introducing numerical diffusion and degrading accuracy near sharp gradients or discontinuities.
Neural approaches such as physics-informed neural networks (PINNs)~\cite{raissi2019physics} represent solutions implicitly via fixed architectures, but suffer from approximation bias, limited extrapolation, and lack of symbolic interpretability.
As a result, both paradigms struggle when exactness, symbolic structure, and global generalization are simultaneously required (Fig.~\ref{fig:hj_comparison}).

\begin{figure}[t]
  \vskip 0.2in
  \begin{center}
    \centerline{\includegraphics[width=\columnwidth]{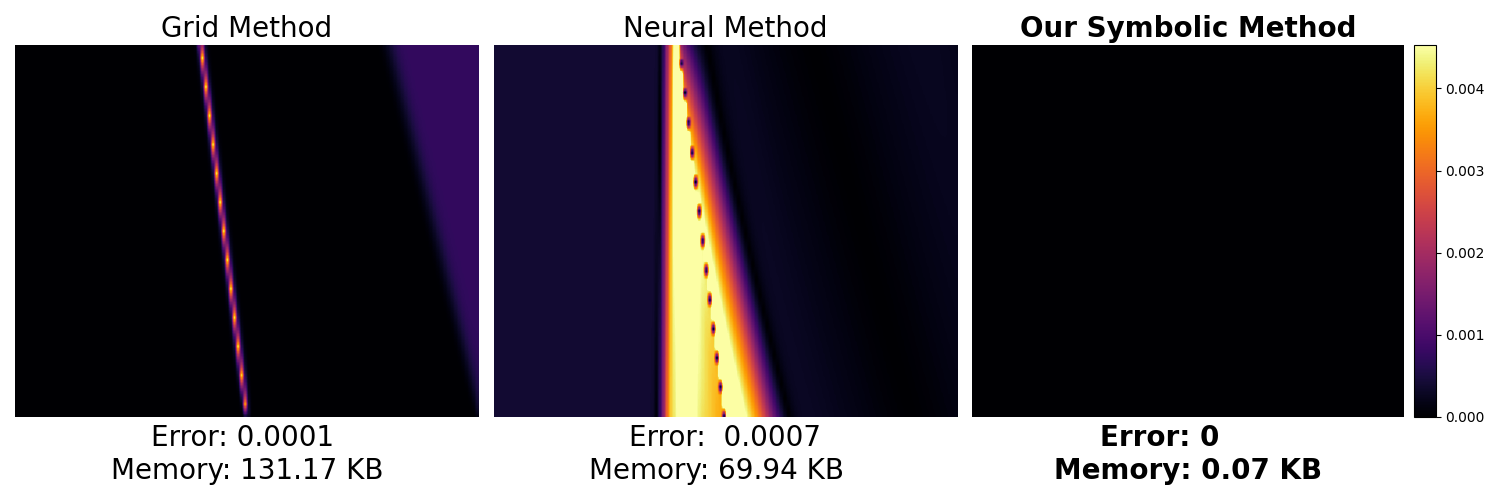}}
    \caption{Comparison of discontinuous solutions of a Hamilton-Jacobi PDE using grid-based, neural network, and our symbolic method. The symbolic approach achieves exact solutions with minimal storage memory, whereas other methods exhibit larger errors and higher storage requirements.}
    \label{fig:hj_comparison}
  \end{center}
  \vspace{-25pt}
\end{figure}

A fundamentally different perspective is to solve PDEs by discovering symbolic expressions. From a computational standpoint, symbolic PDE solution discovery is a discrete combinatorial search problem over expression trees composed of mathematical operators, variables, and constants.
While the formulation closely parallels natural language generation, the present setting poses two key challenges.
First, symbolic expressions are inherently tree-structured, with strict grammatical constraints defined by operator arity and hierarchical relationships, rather than linear token sequences.
Second, the problem is inherently unsupervised: ground-truth expressions are unknown, and supervision is provided only implicitly through PDE and boundary conditions, yielding sparse and delayed learning signals, arriving only after a complete expression has been generated.

Recent work~\cite{wei2024closed} explored symbolic PDE discovery using reinforcement learning (RL) with recurrent neural networks (RNNs). While promising, such recurrent approaches struggle to model the long-range dependencies and hierarchical structure inherent to symbolic mathematics, limiting scalability to complex PDEs.
Transformers naturally address long-range dependencies, yet standard sequence-based Transformers lack inductive biases for tree structure and grammatical validity, and existing tree-based variants rely on externally provided parse trees—an assumption incompatible with symbolic PDE discovery.

In this work, we formulate symbolic PDE discovery as a structured decision-making problem over expression trees and introduce \textbf{SymPlex} (SYMbolic PDE Learning EXplorer), a framework for discovering analytical PDE solutions without access to ground-truth expressions.
At its core is \textbf{SymFormer}, a structure-aware Transformer that defines a fully differentiable policy over symbolic expressions.
SymFormer integrates tree-relative self-attention to model hierarchical dependencies, grammar-constrained autoregressive decoding to enforce syntactic validity, and traversal-aware positional encoding to preserve structural context during generation.
To our knowledge, this is the first Transformer-based architecture for symbolic PDE solution discovery.

Because candidate expressions can only be evaluated after complete generation and because no target solutions are available, training necessarily departs from standard supervised learning for Transformers.
We therefore cast symbolic PDE discovery as an RL problem, training SymFormer with PDE-based rewards and augmenting policy optimization with diversity-aware top-$k$ memory for exploration, imitation of high-reward expressions for stability, and curriculum learning to manage the rapidly growing combinatorial search space as expression complexity increases.

Through extensive experiments, we show that SymPlex reliably discovers exact analytical solutions, including non-smooth solutions that challenge numerical and neural methods.
Furthermore, by treating physical coefficients as symbolic variables, SymPlex recovers parametric solutions that explicitly reveal how solutions depend on underlying parameters, providing direct physical insight and enabling downstream applications such as inverse problems.
Table~\ref{tab:comparison} summarizes the differences between classical numerical methods, neural network–based approaches, and our symbolic framework.

\begin{table*}[t]
\centering
\small
\renewcommand{\arraystretch}{1.0}
\caption{Comparison of PDE solution paradigms. Symbolic solutions offer interpretability, exactness, and parametric generalization.}
\label{tab:comparison}
\resizebox{\textwidth}{!}{%
\begin{tabular}{lccccc}
\toprule
Method & Interpretability & Solution & Accuracy & Generalization & Parametric \\
& & Storage & & & Discovery \\
\midrule
Classical Numerical Methods & Low & High & Approximate & Grid-limited & \xmark \\
Neural Network Methods     & Low & High & Non-guaranteed & Limited & \xmark \\
\textbf{SymPlex (Ours)}     & \textbf{High} & \textbf{Very Low} & \textbf{Exact} & \textbf{Analytic} & \cmark \\
\bottomrule
\end{tabular}}
\end{table*}

In summary, our contributions are:
\begin{enumerate}
    \item We formulate PDE solving as symbolic expression discovery, highlighting exact representation of non-smooth solutions, analytic generalization, and explicit parametric dependence.
    \item We introduce \textbf{SymFormer}, a structure-aware Transformer for symbolic expression generation that incorporates explicit hierarchical inductive biases.
    \item We present \textbf{SymPlex}, an RL framework for PDE-guided symbolic discovery without requiring access to ground-truth solutions, integrating policy optimization with constant refinement, diversity-aware top-$k$ memory, and curriculum learning.
\end{enumerate}

Together, these contributions establish a new paradigm for PDE solution discovery that unifies symbolic computation, hierarchical modeling, and reinforcement learning.

\section{Related Work}
\paragraph{Symbolic PDE and expression discovery.} 
Automatically discovering analytical or symbolic solutions for PDEs has been studied extensively using symbolic regression techniques. Classical methods include genetic programming (GP) and heuristic search strategies \cite{forrest1993genetic, koza1994genetic, schmidt2009distilling, blkadek2019solving, randall2022bingo, jiang2023symbolic}, which explore the space of candidate expressions to minimize residuals or other fitness measures. While effective for small-scale problems, these approaches often struggle with combinatorial search spaces, producing overly complex expressions without corresponding performance gains. 

To improve efficiency and scalability, deterministic symbolic search methods \cite{liang2025finite} systematically combine operators to construct PDE solutions, providing a structured approach without relying on stochastic exploration. More recently, neural networks have been integrated to guide symbolic search via reinforcement learning, such as policy gradients \cite{wei2024closed} or Monte Carlo tree search  \cite{sahoo2018learning, zhang2023deep, dong2024neural}, enabling candidate expressions that satisfy differential constraints.
However, these methods have so far been demonstrated only on very simple PDEs, typically admitting smooth solutions and shallow symbolic structures, limiting their applicability to more complex PDEs. Supervised learning–based symbolic regression methods, pre-trained on large synthetic datasets, can produce expressions in a single forward pass \cite{li2023neural, kamienny2022end, biggio2021neural}, though their performance degrades when the target PDE distribution diverges from the training data.

Recent efforts integrate PINNs with symbolic regression to derive interpretable or approximate analytical solutions. Majumdar et al.~\yrcite{majumdar2022physics} combined neural network representations with a small symbolic basis to encode PDE constraints; while effective for simple PDEs, their fixed-basis formulation limits expressiveness for general expressions. Other works first train a PINN and then fit a symbolic expression to its output~\cite{changdar2024integrating, huang2025physics, das2025physics}, yielding interpretable solutions for simple PDEs but constrained by the PINN’s representation and approximation errors. Consequently, these methods are limited in expressiveness, accuracy, and scalability for long or complex PDE solutions.
\vspace{-.6em}
\paragraph{Tree-structured and grammar-aware Transformers.}  
Transformers have been adapted to hierarchical and structured data, such as trees or graphs, to enable reasoning over compositional representations \cite{wang2019tree, peng2021integrating, hu2021r2d2, zhu2025tamer, fu2025satd}. In symbolic mathematics and program synthesis, grammar-constrained and arity-aware decoding has been used to enforce syntactic correctness during generation \cite{yin2017syntactic, saxton2019analysing, allamanis2018survey}. Existing approaches typically assume fixed parse trees or require full supervision, limiting their applicability for solving PDEs, where the expression tree must be inferred dynamically. Our SymFormer architecture combines tree-relative self-attention, traversal-aware positional encoding, and grammar-constrained autoregressive decoding, providing the inductive biases needed to generate valid, hierarchical, and semantically meaningful symbolic solutions directly from PDEs.

\section{Preliminary}

\subsection{Parametric PDEs}
We consider discovering analytical solutions to \emph{parametric partial differential equations (PDEs)}. Let $\Omega \subset \mathbb{R}^n$ denote a spatial domain with boundary $\partial \Omega$, and $t \in [0,T]$ denote time.  
A general parametric PDE is written as
\begin{equation}\label{eq:param_pde}
\mathcal{F}[u](x,t;\kappa) = f(x,t;\kappa), \quad (x,t) \in \Omega \times [0,T],
\end{equation}
where $\mathcal{L}$ is a differential operator, $f$ a source term, and $\kappa \in \mathbb{R}^p$ denotes physical parameters.  
For well-posedness, the PDE is supplemented with boundary conditions
\begin{equation}
\mathcal{B}[u](x,t;\kappa) = 0, \quad x \in \partial \Omega,
\end{equation}
and, for time-dependent problems, an initial condition
\begin{equation}
u(x,0;\kappa) = u_0(x;\kappa), \quad x \in \Omega.
\end{equation}
Consequently, the solution $u(x,t;\kappa)$ depends explicitly on both the independent variables $(x,t)$ and the parameters $\kappa$.

\subsection{Symbolic Expression Representation}\label{sec:preliminary_symbolic_expression}
We represent solutions $u(x,t;\kappa)$ as structured expressions encoded by \emph{abstract syntax trees (ASTs)}, generated via \emph{prefix traversal} to enable unambiguous, autoregressive construction while preserving hierarchy.

Tokens in the expression belong to three categories:
\vspace{-.5em}
\begin{itemize}
    \item \textbf{Binary operators:} $\mathcal{B} = \{+, -, \times, /\}$, each with arity $\alpha(b) = 2$.
    \item \textbf{Unary operators:} $\mathcal{U} = \{\sin, \cos, \exp, \sqrt, \dots\}$, each with arity $\alpha(u) = 1$.
    \item \textbf{Variables, parameters, and constants:} 
    \[
        \mathcal{T} = \{x_1,\dots,x_n, t, \kappa, \text{const}\},
    \] 
    which are leaf nodes with arity $\alpha(t) = 0$.
\end{itemize}
\vspace{-.5em}
The full vocabulary is $\mathcal{V} = \mathcal{B} \cup \mathcal{U} \cup \mathcal{T}$.  
Token arities define a deterministic grammar: at each step, the valid next tokens are uniquely determined, ensuring every prefix sequence forms a syntactically correct AST.  
This structural constraint supports both symbolic generation and reinforcement learning optimization, guaranteeing well-formed, parametric expressions for $u(x,t;\kappa)$.

\section{SymFormer: Structure-Aware Attention for Symbolic Expressions}\label{sec:symformer}
Symbolic PDE discovery is fundamentally a \emph{structured decision-making problem over trees}, rather than a flat sequence prediction task.  
\textbf{SymFormer} extends the standard Transformer to explicitly incorporate symbolic hierarchy and grammar, enabling differentiable, grammar-preserving generation of expression trees—critical when ground-truth solutions are unavailable 

\begin{figure*}[t]
    \centering
    \includegraphics[width=0.8\textwidth]{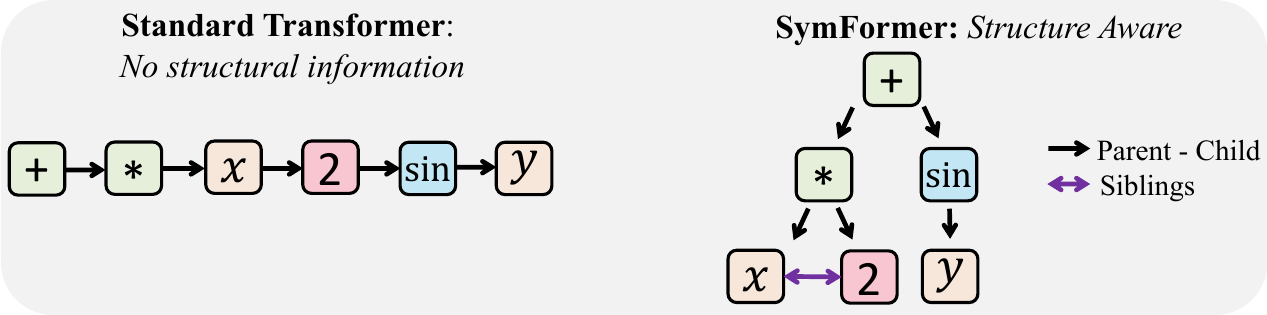}
    \caption{
    \textbf{Left}: Standard attention ignores tree structure. \textbf{Right}: SymFormer explicitly models hierarchical relations.}
    \label{fig:symformer_attention}
\end{figure*}

\subsection{Tree-Relative Self-Attention for Hierarchical Reasoning}
Semantic meaning in symbolic expressions arises from hierarchical relationships: operators depend on operands, which may be non-adjacent in prefix notation.  
Standard self-attention, based on token identity and linear position, cannot capture such hierarchical dependencies effectively. 

SymFormer introduces \emph{tree-relative self-attention}, which conditions attention scores on structural relations inferred from the partial AST.  
Each token pair $(i,j)$ is assigned a discrete relation:
\[
r_{ij} =
\begin{cases}
0, & i=j \quad (\text{self})\\
1, & p[j]=i \quad (\text{parent})\\
2, & p[i]=j \quad (\text{child})\\
3, & p[i]=p[j]\neq -1 \quad (\text{sibling})\\
4, & i \text{ is ancestor of } j\\
5, & \text{otherwise},
\end{cases}
\]
where $p[\cdot]$ denotes the parent index in the inferred tree.
Each relation type is associated with a learnable embedding $R_{r_{ij}}$, which augments the attention computation:
\[
\mathrm{Attn}(Q,K,V)
=
\mathrm{softmax}\!\left(
\frac{Q (K+R)^\top}{\sqrt{d}}
\right)V.
\]

This allows operators to naturally attend to their operands, operands to their parents, and siblings to each other, enabling hierarchical information flow even in deep and partially completed expression trees.  
Unlike prior tree-based Transformers with fixed parse trees, SymFormer infers tree structure dynamically during autoregressive generation, which is critical for symbolic PDE discovery where the expression tree is unknown a priori.

\subsection{Traversal-Aware Positional Encoding}
To distinguish nodes in different subtrees that occupy similar structural roles, SymFormer retains sinusoidal positional encodings along the prefix traversal:
\[
\mathbf{X} \leftarrow \mathbf{X} + \mathrm{PosEnc}(\texttt{prefix\_position}).
\]
This preserves information about generation order and relative depth without introducing additional tree-specific parameters.  
Together with tree-relative self-attention, this allows structurally similar but contextually distinct nodes to be represented distinctly, supporting robust modeling of expressions with repeated substructures.

\subsection{Grammar-Constrained Autoregressive Generation}
Expression generation is autoregressive, with dynamic grammar and depth constraints enforced at each step.  
Each token is only expanded if the remaining depth budget allows its minimal required subtree height, ensuring that the generated expressions are syntactically valid and semantically meaningful.
The autoregressive generation proceeds as follows:

\begin{enumerate}
    \item Conditions on the previously generated prefix,
    \item Reconstructs the partial AST from arity constraints,
    \item Computes tree-relative self-attention over the inferred structure,
    \item Samples the next token from the grammar- and depth-constrained action space.
\end{enumerate}
Dynamic constraints ensure that leaf nodes are restricted to variables or constants, internal nodes satisfy operator arities, degenerate sub-expressions (e.g., $x-x$, $x/x$) are filtered, and each token is only expanded if it fits within the remaining depth budget.


This separation between symbolic validity, subtree feasibility, and representation learning allows SymFormer to focus on meaningful compositional reasoning.  
Depth-aware expansion is particularly beneficial for PDE symbolic expression discovery, where repeated sub-expressions and nested operator hierarchies are common.  
By explicitly modeling hierarchy, grammar, and depth feasibility, SymFormer faithfully generates deeper and more complex expressions than prior sequential or tree-based symbolic models.

In contrast to the RNN-based symbolic solver~\cite{wei2024closed}, which encode structure implicitly over flat sequences, and prior tree-based Transformers~\cite{wang2019tree, peng2021integrating, hu2021r2d2, zhu2025tamer}, which rely on fixed or externally provided parse trees, SymFormer represents hierarchy explicitly, enforces depth- and grammar-aware constraints, and leverages Transformer self-attention to model long-range dependencies across repeated subtrees.  

\subsection{Structure-Conditioned Expressive Power of SymFormer}
\label{sec:express_symformer}

In symbolic PDE discovery, solutions are structured expressions rather than flat token sequences.  
Standard sequence models, such as RNNs or vanilla Transformers, reason only over linear prefixes, which can conflate distant sub-expressions and limit generalization over hierarchical structure.
SymFormer explicitly models tree hierarchy and grammar via tree-relative self-attention and grammar-constrained generation.  
This allows the model to condition its next-token decisions on the structure of the partially generated expression tree rather than just the linear token order.

\begin{theorem}[Informal]
Any grammar-compatible next-token policy that depends only on the structure of a partial abstract syntax tree of bounded depth can be represented by SymFormer.
Equivalently, SymFormer can realize any decision rule defined over symbolic tree states, independent of the particular linear prefix representation.
\end{theorem}
A formal statement and rigorous proof are provided in
Appendix~\ref{appen:express_symformer}.  
This result highlights two key strengths of SymFormer:
(i) Tree-relative attention enables direct information flow
between operators, operands, and subtrees, faithfully capturing the hierarchical dependencies inherent in compositional symbolic expressions. (ii) Dynamic grammar constraints ensure that every generated token preserves syntactic validity, thereby restricting the model’s action space to well-formed symbolic expressions and enabling safe exploration during unsupervised PDE discovery.

Together, these properties provide a theoretical justification for SymFormer’s architecture: it has sufficient expressive power to implement any tree-structured token-selection policy, enabling discovery of complex, hierarchical symbolic PDE solutions that cannot, in general, be represented by sequence-based symbolic models.

\section{Reinforcement Learning for Symbolic PDE Solutions}\label{sec:rl_training}

\textbf{SymPlex} trains SymFormer to discover symbolic PDE solutions without ground-truth expressions.  
Unlike conventional supervised sequential training, PDE solution discovery provides feedback only through PDEs and boundary conditions, yielding sparse, delayed rewards defined over entire expression trees.
We formulate symbolic PDE discovery as a reinforcement learning problem: SymFormer acts as a stochastic policy over symbolic trees, rewards are derived solely from PDE and boundarg conditions, and learning proceeds without any target expressions.

Algorithm~\ref{alg:training} summarizes the pipeline:
\vspace{-.6em}
\begin{itemize}
    \item SymFormer as a structure-aware autoregressive policy,
    \item Gradient-based optimization of continuous constants,
    \item A diversity-aware top-$k$ memory with imitation,
    \item Entropy-regularized policy optimization.
\end{itemize}
\vspace{-.6em}
\begin{algorithm}[tb]
  \caption{Training SymPlex for Symbolic PDE Solutions}
  \label{alg:training}
  \begin{algorithmic}
    \STATE {\bfseries Input:} SymFormer policy $\pi_\theta$, top-$k$ memory, curriculum stages $S$, iterations per stage $T_s$, number of sequences $N$, learning rate $\eta$
    \FOR{curriculum stage $s = 1$ {\bfseries to} $S$}
      \FOR{iteration $t = 1$ {\bfseries to} $T_s$}
        \STATE Sample $N$ sequences $\{\text{seq}_i\} \sim \pi_\theta$
        \FOR{each $\text{seq}_i$}
          \STATE Optimize constants $c_i$ for PDE residual
          \STATE Compute reward $R_i$ using PDE residual + BC loss
        \ENDFOR
        \STATE Update top-$k$ memory with diversity-aware selection
        \STATE Compute policy loss $\mathcal{L}_\text{policy}$ and imitation loss $\mathcal{L}_\text{imit}$
        \STATE Update policy parameters: $\theta \gets \theta - \eta \nabla_\theta (\mathcal{L}_\text{policy} + \mathcal{L}_\text{imit})$
      \ENDFOR
    \ENDFOR
  \end{algorithmic}
\end{algorithm}

\subsection{PDE-Aware Reward}
Let $T \sim \pi_\theta$ be a generated symbolic expression tree and $c$ the continuous constants optimized for $T$.  We define the PDE-based reward as
\[
\mathcal{R}(T,c) = \frac{1}{1 + \sqrt{\mathcal{E}(T,c)}}, \ 
\mathcal{E}(T,c) = \left(\mathcal{L}_{\text{PDE}} + \lambda_{\text{BC}} \mathcal{L}_{\text{BC}}\right)[u_T],
\]
where $u_T(x;c)$ denotes the function obtained by evaluating $T$ with constants $c$, $\lambda_{\text{BC}}>0$ is a regularization parameter, and $\mathcal{L}_{\text{PDE}}$ and $\mathcal{L}_{\text{BC}}$ are the PDE and boundary residuals measured as $L^2$ norms, similar to PINNs \cite{raissi2019physics}.

Constants $c$ are optimized separately via gradient descent prior to reward evaluation, decoupling discrete structure search from continuous parameter fitting.
This design ensures that the reward reflects the expressive adequacy of the symbolic structure rather than sensitivity to constant initialization, thereby reducing reward variance and stabilizing policy learning.

\afterpage{%
\begin{table*}[t]
\centering
\caption{Summary of tested differential equations with problem settings.}\label{tab:summary_pdes}
\renewcommand{\arraystretch}{1.4}
\resizebox{\textwidth}{!}{%
\begin{tabular}{c c l c}
\hline
Problem & Name & PDE & IC / Source Term \\ \hline
\multirow{3}{*}{Smooth Problem}
 & Poisson & $ -u_{xx} - u_{yy} = f(x,y)$ & $f(x,y) = -12 x^2 - 4.8 y^2$ \\
 & Advection & $u_t + u_x + u_y = 0$ & $u(x,y,0) = \exp\left(-\left(x^2+y^2\right)/0.5\right)$ \\
 & Heat  & $u_t - u_{xx} - u_{yy} = 0$ & $u(x,y,0)=\sin(x)\cos(y)$ \\ \hline
\multirow{2}{*}{Non-Smooth Problem}
 & Eikonal & $u_t+ \sqrt{u_x^2+u_y^2} = 0$  & $u(x,y,0) = \sqrt{x^2+y^2}$ \\
 & Burgers & $u_t- \frac{1}{2}\left(u_x^2+u_y^2\right) = 0$ &  $u(x,y,0)=|x|+|y|$ \\ \hline
\multirow{2}{*}{Parametric Solution}
 & Advection & $u_t + \kappa(u_x+u_y) = 0$ & $u(x,y,0)=\begin{cases}
1 - |x| - |y|, & \text{if } |x| + |y| \le 1, \\
0, & \text{if } |x| + |y| > 1.
\end{cases}$ \\
 & Heat  & $u_t - \kappa(u_{xx} + u_{yy}) = 0$ & $u(x,y,0)=\exp(-x)\exp(-y)$ \\ \hline
\end{tabular}}
\renewcommand{\arraystretch}{1} 
\end{table*}}
\subsection{Policy Loss with Entropy and Imitation}
At each node of a generated tree $T$, the policy predicts a distribution over the next token, with node-wise entropy
\[
\mathcal{H}(p) = - \sum_{v \in \mathcal{V}} p(v) \log p(v),
\]
encouraging exploration and preventing premature convergence in the sparse, combinatorial symbolic search space.

The RL policy loss is
\[
\mathcal{L}_{\text{policy}} = - \mathbb{E}_{T\sim\pi_\theta} \Big[ \log \pi_\theta(T) \cdot \mathcal{R}(T,c) \cdot w_\text{depth} \Big] 
- \lambda_\text{ent} \mathcal{H}[\pi_\theta],
\]
where $w_\text{depth}$ penalizes excessively deep trees. The entropy term ensures that the policy continues to explore diverse tree structures, mitigating early collapse to suboptimal patterns.

In addition, an imitation loss reinforces previously discovered high-reward trees stored in the top-$k$ memory:
\[
\mathcal{L}_{\text{imit}} = \sum_{(T^*, c^*) \in \text{top-}k} w_R \, \text{NLL}(\pi_\theta(T^*) \mid c^*),
\]
where $w_R$ is a reward-based weight and NLL denotes the negative log-likelihood of reproducing stored trees. This improves sample efficiency 
by encouraging reuse of successful structures while maintaining exploration.

The total loss for updating $\theta$ is
\[
\mathcal{L}_{\text{total}} = \mathcal{L}_{\text{policy}} + \lambda \, \mathcal{L}_{\text{imit}},
\]
where $\lambda$ balances imitation and policy objectives.

\subsection{Theoretical Guarantee for Symbolic Recovery}
The reward $\mathcal{R}$ is designed so that its global maxima correspond to the PDE solution whenever it is representable in the symbolic hypothesis class.  
This ensures that SymPlex can recover exact solutions under ideal optimization while providing probabilistic guarantees for near-optimal policies.

\begin{theorem}[Symbolic Recovery by SymPlex]
\label{thm:symplex_recovery}
Let $\pi_\theta$ denote the learned SymPlex policy.  

\noindent\textbf{(i) Exact Recovery:} If $\pi_\theta$ is globally optimal, there exists a tree $T^*$ in its support with constants $c^*$ such that
\[
u_{T^*}(x;c^*) = u^*(x) \quad \text{a.e. in } \Omega.
\]

\noindent\textbf{(ii) Near-Optimal Probabilistic Recovery:} For any policy $\pi_\theta$ with expected reward
\[
\mathbb{E}_{T \sim \pi_\theta}[\mathcal{R}(T)] \ge 1 - \epsilon, \quad \epsilon \in [0,1),
\]
there exists at least one tree $T$ in the support with constants $c_T^*$ satisfying
\[
\inf_c \mathcal{E}(T,c) \le \frac{\epsilon^2}{(1-\epsilon)^2}.
\]
\end{theorem}

A formal statement and proof are provided in Appendices~\ref{appen:global_optimality} and~\ref{appen:symplex_recovery}.  
This theorem establishes that SymPlex recovers exact symbolic solutions under globally optimal policies, and that high expected reward under near-optimal policies guarantees the presence of approximate solutions with small PDE residuals.

\subsection{Curriculum Learning}
Symbolic exploration suffers from an exponentially growing search space as expression depth, number of variables, or vocabulary size increases, making complex PDEs with deep expressions or high-dimensional domains challenging. Approaches that treat variables independently \cite{wei2024closed} are generally infeasible, since even separable solutions do not guarantee variable-wise separability in the PDE itself.

To address this, we adopt a curriculum learning strategy that gradually increases problem complexity across three stages, using a single shared SymFormer model:
\begin{itemize}
    \item \textbf{Stage 1 (Spatial variables):}  
    Train on spatial variables $x$ only, with $t=0$ and $\kappa$ fixed. The reward is based solely on the initial condition, allowing the model to learn meaningful spatial expressions first.

    \item \textbf{Stage 2 (Full PDE with fixed $\kappa$):}  
     Introduce $t$ and train on the full PDE residual, boundary conditions, and initial condition, keeping $\kappa$ fixed. Stage 1 expressions serve as priors, capturing temporal dynamics while retaining the previously discovered  spatial structure.

    \item \textbf{Stage 3 (Parametric PDE solution):} 
    Include parameters $\kappa$ as variables, training the model to produce the full parametric solution $u(x,t;\kappa)$. Stage 1–2 expressions serve as priors, facilitating efficient of the larger search space.
\end{itemize}

This curriculum strengthens the training signal gradually, letting the model first discover useful sub-expressions and then compose them in more complex settings, without assuming variable-wise separability.

\begin{table*}[ht]
\centering
\caption{Comparison of PDE solvers on MSE and SRR}\label{tab:comparison_methods}
\renewcommand{\arraystretch}{1.3} 
\resizebox{\textwidth}{!}{%
\begin{tabular}{c|cc|cc|cc|cc|cc}
\hline
\multirow{2}{*}{Problem} & \multicolumn{2}{c|}{\textbf{SymPlex}} & \multicolumn{2}{c|}{SSDE} & \multicolumn{2}{c|}{FEX} & \multicolumn{2}{c|}{PINN+DSR} & \multicolumn{2}{c}{KAN} \\
 & MSE ($\downarrow$) & SRR ($\uparrow$) & MSE ($\downarrow$) & SRR ($\uparrow$) & MSE ($\downarrow$) & SRR ($\uparrow$) & MSE ($\downarrow$) & SRR ($\uparrow$) & MSE ($\downarrow$) & SRR ($\uparrow$) \\ \hline
Poisson & 0 & \textbf{100\%} & $1.24\times10^{-13}$ & 5\% & $1.04\times10^{-15}$& 100\% & $3.17\times 10^{-3}$ &  0\% & $9.25\times 10^{-4}$ & 0\%\\
Advection & 0 & \textbf{100\%} & $1.98\times10^{-1}$ & 0\% & $2.37\times10^{-2}$ & 0\% & $1.85\times 10^{-2}$ & 0\%  & $3.01\times 10^{-3}$ & 0\% \\
Heat & 0 & \textbf{100\%} & $6.12 \times 10^{-2}$ & 0\% & $3.73\times10^{-2}$ & 0\% & $4.30\times 10^{-2}$ & 0\% & $4.87\times 10^{-2}$ & 0\%\\
Eikonal & 0 & \textbf{100\%} & $3.43\times10^{-1}$ & 0\% &$1.26\times10^{-2}$ & 0\% & $3.27\times 10^{-2}$ & 0\% & $5.15\times 10^{-3}$ & 0\%\\
Burgers & 0 & \textbf{100\%} &$1.36\times10^{-32}$ & 0\%  & $7.59\times10^{-15}$& 10\% & $3.50\times 10^{-6}$ &  0\% & $3.85\times 10^{-4}$ & 0\%\\
Parametric Advection & 0 & \textbf{100\%} & $1.87\times10^{-1}$ & 0\% &$3.32 \times10^{-2}$& 0\% & $5.40\times 10^{-2}$ & 0\% & $2.46\times 10^{-2}$ & 0\%\\
Parametric Heat & 0 & \textbf{100\%} & $4.11\times10^{-1}$ &0\%  &$3.29\times10^{-3}$ & 0\% & $7.97\times 10^{-2}$ & 0\%  & $1.41\times 10^{-2}$ & 0\% \\ 
\hline
\end{tabular}}
\renewcommand{\arraystretch}{1} 
\end{table*}

\section{Experiments}
We experimentally evaluate the proposed \textbf{SymPlex} framework. All experiments are conducted on a single NVIDIA GV100 (TITAN V) GPU. Implementation details are provided in Appendix~\ref{appen:symformer_impl}.

\subsection{Problem Settings}
Table~\ref{tab:summary_pdes} summarizes the PDE problems considered, covering three categories. Non-smooth and parametric solutions are included to highlight the advantages of symbolic approach. 
\begin{itemize}
    \item \textbf{Smooth Solutions:} Benchmark PDEs with smooth solutions, including the Poisson, advection, and heat equations, are solved on a two-dimensional spatial domain. These tests evaluate the basic ability of SymPlex to recover PDE solutions.
    \item \textbf{Non-smooth Solutions:} To demonstrate the capability of symbolic solutions to reduce numerical errors near kinks, we consider two Hamilton-Jacobi PDEs with non-smooth solutions.
    \item \textbf{Parametric Solutions:} We assess whether symbolic expressions capture the parametric dependence $\kappa$ in \eqref{eq:param_pde} using advection and heat equations, including some non-smooth solutions, and learn both the solution and its dependence on velocity and diffusion coefficients.
\end{itemize}

\subsection{Baselines}
We compare SymPlex against existing approaches for symbolic PDE solutions:
\begin{itemize}
    \item \textbf{SSDE}~\cite{wei2024closed}: RNN-based symbolic regression for PDEs trained with reinforcement learning.
    \item \textbf{FEX}~\cite{liang2025finite}: Deterministic tree-based symbolic search by combining operators. 
    \item \textbf{PINN+DSR}: Symbolic regression via DSR~\cite{petersen2019deep} applied to solutions obtained by PINNs~\cite{raissi2019physics}.
    \item \textbf{KAN} \cite{liu2024kan}: Symbolic regression using the Kolmogorov–Arnold representation for constructing structured PDE solutions.
\end{itemize}

\subsection{Evaluation Metrics}
We measure both numerical and symbolic accuracy:
\vspace{-0.4em}
\begin{itemize}
    \item \textbf{MSE:} Mean squared error between predicted and true solutions.
    \item \textbf{Symbolic Recovery Rate (SRR):} Fraction of expressions that fully recover the true solution, defined as those whose simplified skeleton matches the true solution and whose MSE after constant optimization is below $10^{-8}$, based on~\cite{wei2024closed}.
\end{itemize}

\subsection{Results}
Table~\ref{tab:comparison_methods} reports the MSE and SRR for all methods across the PDEs summarized in Table~\ref{tab:summary_pdes}, while Table~\ref{tab:results_symbolic} presents the symbolic solutions recovered by SymPlex; baseline expressions are provided in Appendix~\ref{appen:results_symbolic}.

SymPlex achieves exact recovery (SRR $=100\%$) with zero numerical error across all PDEs, including nonlinear, discontinuous, and parametric cases. In contrast, existing methods show limited symbolic recovery.
SSDE~\cite{wei2024closed} occasionally attains low MSE but fails to recover correct symbolic structures, particularly for nonlinear and parametric cases, highlighting the difficulty of long-horizon symbolic search with recurrent policies.
FEX~\cite{liang2025finite} succeeds on simple linear problems but degradesas expressions become deeper or more compositional.
Neural and regression-based pipelines (PINN+DSR and KAN) achieve moderate numerical accuracy but never recover exact symbolic solutions, reflecting approximation bias and the lack of explicit structural constraints.

These results demonstrate that SymPlex provides a practical and reliable approach for discovering closed-form PDE solutions, in settings requiring exact symbolic recovery and parametric generalization. Additional results are provided in Appendix~\ref{appen:add_results}.

\begin{table*}[ht]
\centering
\caption{True solution and symbolic expression attained by SymPlex}\label{tab:results_symbolic}
\renewcommand{\arraystretch}{1.4}
\resizebox{\textwidth}{!}{%
\begin{tabular}{c c l c}
\hline
 Name & Analytic Solution & Predict Solution \\ \hline
 Poisson & $x^4 + 1.2y^4$ & \texttt{((y )\^{}4 * 1.2) - (-x\^{}4)} \\
 Advection &  $\exp(-((x-t)^2+(y-t)^2)/0.5)$ & \texttt{exp(-(2.0*((x - t)\^{}2)))*exp(((t + (-y))*(-2.0*(t + (-y)))))}\\
 Heat  & $\sin(x)\cos(y)\exp(-2t)$ & \texttt{(sin(x) * (exp((-2.0*(k * t))) * (0.99 * cos(y))))} \\ \hline
 Eikonal  & $\begin{cases} \sqrt{x^2+y^2} - t, & \sqrt{x^2+y^2} \geq t \\ 0, & \sqrt{x^2+y^2} < t \end{cases}$ & \texttt{max[0,(sqrt((((1.0*x)*x) + (y*y))) - ((t))]} \\
 Burgers  & $|x|+|y|+t$ & \texttt{(abs(y) + ((-0.0 + abs(x)) - (-0.2436 * (t / 0.2436))))} \\ \hline
 Parametric Advection & $\max\{1 - |x-\kappa t| - |y-\kappa t|, 0\}$ & \texttt{max[(1-(abs(((k * t) - y))+(1.0*abs((x - (abs(t)*k))))))]} \\
 Parametric Heat & $\exp(-x)\exp(-y)\exp(2\kappa t)$ & \texttt{exp((-((k*(-t)))-(x)))* (exp(-(y)) * exp(((1.0*k) * t)))} \\ \hline
\end{tabular}}
\renewcommand{\arraystretch}{1}
\end{table*}

\subsection{Comparison with Standard PDE Solvers}
To further evaluate the proposed symbolic PDE solver \textbf{SymPlex}, we compare it against two representative baselines: the grid-based WENO scheme~\cite{jiang2000weighted} and neural network-based PINNs~\cite{raissi2019physics} on Burgers' equation with a non-smooth solution. 
We report the following metrics: Computation time, the total time to obtain the solution; inference time, the time to evaluate the solution at 100 new spatiotemporal points (requiring recomputation for WENO, but performed post-training for PINN and SymPlex); peak memory, the maximum memory used during computation; and solution storage, the memory required to store the obtained solution for later use.
Table~\ref{tab:comparison} summarizes key performance metrics, while qualitative results of the numerical solutions are shown in Figure~\ref{fig:hj_comparison}. 

\begin{table}[h]
\centering
\caption{Performance comparison of SymPlex with standard baselines on Burgers' equation. 
}
\label{tab:comparison}
\resizebox{\columnwidth}{!}{%
\begin{tabular}{lccc}
\toprule
\textbf{Metric} & \textbf{WENO} & \textbf{PINN} & \textbf{SymPlex} \\
\midrule
Applicability Domain & Grid-only & Trained domain only & Global \\
Error ($\downarrow$) & 1.0e-4 & 7.0e-4 & \textbf{0.0} \\
Computation Time (s) & 0.35 & 1105.06 & 642.01 \\
Inference Time (s) & 0.26 & 0.073 & \textbf{6.27e-5} \\
Peak Memory (MB) & 0.06 & 52.92 & 251.37 \\
Solution Storage (MB) & 131.17 & 69.94 & \textbf{0.07} \\
\bottomrule
\end{tabular}}
\end{table}

The results highlight a key potential of symbolic representations. Unlike WENO and PINN, which show large error spikes near kinks, SymPlex accurately recovers the solution in non-smooth regions. This illustrates that, when the symbolic vocabulary includes relevant functions and operators, symbolic expressions can capture non-smoothness with high fidelity, and produce a global solution allowing evaluation beyond the computational domain. Earlier results on parametric solution recovery further illustrate the advantage of symbolic expressions.

Although training SymPlex requires a Transformer-based architecture and RL, resulting in higher memory usage and longer training times,  inference is fast. Additionally,  storing the solution is highly memory-efficient as only a compact symbolic expression is needed. Overall, these results demonstrate the stength and potential of symbolic PDE solvers to provide readable, global, high-fidelity solutions where traditional numerical or neural methods face challenges.

\section{Conclusion}
We presented \textbf{SymPlex}, a reinforcement learning framework for discovering analytical PDE solutions via symbolic expressions. By combining a structure-aware Transformer (\textbf{SymFormer}) with grammar-constrained generation and curriculum learning, SymPlex can reliably generate valid, hierarchical expressions, recover non-smooth solutions exactly, and reveal parametric dependencies. This approach provides a scalable and principled paradigm for symbolic PDE discovery, unifying symbolic computation, hierarchical modeling, and reinforcement learning.

While SymPlex demonstrates improved performance compared to existing symbolic PDE methods, it remains an early-stage approach with several open challenges. High-dimensional PDEs are still difficult due to the combinatorial growth of the search space, and rigorous analysis of convergence and accuracy is required to scale symbolic discovery to more complex scientific problems. 
Another important direction for future work is the selection of the symbolic vocabulary: in practical PDE problems, choosing suitable operators and functions may be nontrivial and require substantial mathematical understanding of the PDE and its solution.

\section*{Acknowledgment}
This material is based upon work supported by the Defense Advanced Research Projects Agency (DARPA) under Agreement No. HR00112590074.

\section*{Impact Statement}
This work introduces a new paradigm for symbolic PDE solution discovery, with potential impact in both scientific computing and mathematical understanding of differential equations. It also demonstrates a novel extension of Transformer architectures to structured, symbolic reasoning. We do not anticipate any direct ethical or societal risks.




\bibliography{mybib}
\bibliographystyle{icml2026}

\newpage
\include{appendix}


\end{document}

%% file: custom_notation.tex
\usepackage{amsmath,amssymb, mathtools, amsthm, amsfonts,bm}

\newcommand{\cE}{\mathcal{E}}

\newcommand{\cU}{\mathcal{U}}

\newcommand{\bR}{\mathbb{R}}

\newcommand{\dV}[1][u]{\cU\left(\Omega;\bR^{d_v}\right)}
\newcommand{\dVt}[1][u]{\cU\left(\left[0,\infty\right)\times\Omega;\bR^{d_v}\right)}

%% file: appendix.tex
\appendix
\onecolumn

\section{Implementation Details of SymFormer}
\label{appen:symformer_impl}
This section provides a detailed description of the SymFormer architecture used for symbolic PDE discovery, complementing the high-level discussion in Section~\ref{sec:symformer}. It covers the token vocabulary, hierarchical attention mechanism, autoregressive generation procedure, and hyperparameter settings.

\subsection{Token Vocabulary and Arity}
SymFormer operates over a fixed vocabulary $\mathcal{V}$ comprising variables, constants, and unary and binary operators:
\begin{itemize}
    \item Variables: $x_0, x_1, \dots, x_{n-1}$, $t$, optionally including $\kappa$ for parametric problems.
    \item Constants: trainable `\texttt{const}` tokens.
    \item Unary operators: \texttt{neg, relu, sqrt, square, ...}.
    \item Binary operators: $+, -, *, /$.
\end{itemize}

Each token $t \in \mathcal{V}$ is associated with an \emph{arity} function $\text{arity}(t)$, defining the number of children in the expression tree:
\[
\text{arity}(t) =
\begin{cases}
0, & t \in \text{Variables or Constants} \\
1, & t \in \text{Unary operators} \\
2, & t \in \text{Binary operators}.
\end{cases}
\]

\subsection{Tree-Relative Multi-Head Self-Attention}
To capture hierarchical dependencies in symbolic expressions, SymFormer extends standard multi-head attention with \emph{tree-relative embeddings}. Given a partially generated prefix sequence, the AST is reconstructed dynamically based on token arities. The relation between any two tokens $i$ and $j$ is categorized as:
\[
r_{ij} =
\begin{cases}
0, & i=j \quad (\text{self})\\
1, & p[j]=i \quad (\text{parent})\\
2, & p[i]=j \quad (\text{child})\\
3, & p[i]=p[j]\neq -1 \quad (\text{sibling})\\
4, & i \text{ is ancestor of } j\\
5, & \text{otherwise},
\end{cases}
\]
where $p[\cdot]$ denotes the parent index in the inferred tree.  
Each relation type $r_{ij}$ is associated with a learnable embedding $R_{r_{ij}}$, which are initialized randomly and optimized jointly with all other model parameters during training. 

For a partially generated sequence, the corresponding AST is reconstructed dynamically based on token arities. A relation matrix $R$ is then formed by looking up the embedding of each pairwise relation $r_{ij}$. This matrix is added to the projected keys in multi-head attention:
\[
\mathrm{Attn}(Q,K,V) = \mathrm{softmax}\Bigg(\frac{Q(K+R)^\top}{\sqrt{d_\mathrm{head}}}\Bigg)V.
\]
By directly incorporating $R_{r_{ij}}$ into the attention computation, the model is able to condition attention scores on the syntactic roles of tokens within the tree, enabling operators to attend preferentially to their operands, siblings, and ancestors. This approach preserves the full differentiability of the Transformer while explicitly embedding symbolic structure, allowing the model to learn the relative importance of different hierarchical relationships during training.

\subsection{Traversal-Aware Positional Encoding}
While tree-relative attention captures hierarchical relations, nodes occupying similar structural roles in different subtrees remain ambiguous.  
SymFormer applies sinusoidal positional encodings along the prefix traversal order:
\[
\mathbf{X} \leftarrow \mathbf{X} + \mathrm{PosEnc}(\texttt{prefix\_position}),
\]
ensuring distinguishability of structurally similar but contextually distinct nodes.

\subsection{Grammar-Constrained Autoregressive Generation}
Expression generation is performed autoregressively with dynamic grammar constraints:
\begin{enumerate}
    \item The partial AST is reconstructed from token arities.
    \item Tree-relative attention is computed over the current sequence.
    \item The next token is sampled only from the grammar-valid action space:
    \begin{itemize}
        \item Leaf nodes $\rightarrow$ variables or constants.
        \item Internal nodes $\rightarrow$ operators consistent with arity.
        \item Degenerate expressions (e.g., $x-x$, $x/x$) are filtered through semantic resampling.
    \end{itemize}
\end{enumerate}
This ensures all generated sequences are syntactically valid while maintaining differentiability for reinforcement learning objectives.

\subsection{Expression Evaluation and Constant Optimization}
Given a prefix sequence $T=[t_1,\dots,t_L]$ and associated constants $c$, expressions are evaluated recursively. Constants are broadcast across the batch to preserve gradient flow for optimization. This enables end-to-end differentiable training of both the symbolic structure and associated parameters.

\subsection{Model Architecture}
SymFormer comprises a stack of $N$ decoder layers, each containing:
\begin{enumerate}
    \item \textbf{Embedding Layer:} Token embeddings $\mathbf{e}_t \in \mathbb{R}^{d_\mathrm{model}}$.
    \item \textbf{Traversal-Aware Positional Encoding:} Applied to $\mathbf{e}_t$ as described above.
    \item \textbf{Tree-Relative Multi-Head Attention:} Incorporates relation embeddings $\mathbf{R}_{r_{ij}}$ in multi-head attention.
    \item \textbf{Feed-Forward Network:} Two-layer position-wise MLP with ReLU activations.
    \item \textbf{Layer Normalization and Residual Connections:} Applied to stabilize training and propagate hierarchical features.
    \item \textbf{Decoder Output Layer:} Maps the last token's hidden representation to the vocabulary distribution.
\end{enumerate}

This design explicitly encodes symbolic hierarchy, enforces grammar, and allows differentiable autoregressive generation without exposing implementation-specific details.

\subsection{Sequence Sampling and Entropy Regularization}
During generation, SymFormer controls tree depth and applies semantic filters:
\begin{itemize}
    \item Maximum tree depth $d_\mathrm{max}$ prevents excessively large expressions.
    \item Degenerate or all-constant subtrees are resampled.
    \item Node-wise entropy is computed to regularize the RL policy.
\end{itemize}

\subsection{Hyperparameters}
All experiments used a fixed set of hyperparameters:
\begin{itemize}
    \item Embedding dimension: $d_\mathrm{model}=64$
    \item Feed-forward hidden dimension: $128$
    \item Number of relation types: $6$
    \item Decoder layers: $4$
    \item Multi-head attention: $8$ heads
    \item Maximum generation depth: $7$–$10$
\end{itemize}
These values were chosen to balance representational capacity with computational efficiency across all symbolic PDE discovery tasks.

\section{Reinforcement Learning Training Details for Symbolic PDE Discovery}
\label{appen:rl_training}

Our RL-based symbolic PDE discovery uses a tree-structured Transformer decoder to generate candidate expressions, optimize their constants, and select top-k diverse solutions. The training procedure is outlined as follows:

\begin{itemize}
    \item \textbf{Sequence Sampling:} At each training epoch, SymFormer samples 64 candidate token sequences from the Transformer policy. Sampling is performed under grammar constraints with a controlled maximum tree depth and a fixed set of allowed variables.

        \item \textbf{Loss computation:} 
SymFormer is trained by minimizing an energy-based objective defined on the generated symbolic expression $u_T$:
\[
\mathcal{E}(T,c)
= \big(\mathcal{L}_{\mathrm{PDE}} + \lambda_{\mathrm{BC}} \mathcal{L}_{\mathrm{BC}}\big)[u_T],
\]
where $\mathcal{L}_{\mathrm{PDE}}$ and $\mathcal{L}_{\mathrm{BC}}$ denote the PDE and boundary-condition residuals, respectively.  
Both residuals are approximated using Monte Carlo sampling at each training epoch:
\[
\mathcal{L}_{\mathrm{PDE}}[u_T] = \frac{1}{N_{\mathrm{PDE}}} \sum_{i=1}^{N_{\mathrm{PDE}}} \big( \mathcal{F}[u_T](x_i,t_i;\kappa_i) - f(x_i,t_i;\kappa_i) \big)^2,
\]
\[
\mathcal{L}_{\mathrm{BC}}[u_T] = \frac{1}{N_{\mathrm{BC}}} \sum_{i=1}^{N_{\mathrm{BC}}} \big( \mathcal{B}\left[u_T\right](x^b_i,t^b_i;\kappa^b_i) \big)^2,
\]
where $(x_i,t_i)$ are sampled uniformly from the interior of the domain, $(x^b_i,t^b_i)$ are sampled uniformly from the domain boundary, and for parametric PDEs, $\kappa_i$ and $\kappa^b_i$ are sampled uniformly from the range of parameter values used during training. For time-dependent PDEs, we additionally include the initial condition loss:
\[
\mathcal{L}_{\mathrm{IC}}[u_T] =
\frac{1}{N_{\mathrm{IC}}}
\sum_{i=1}^{N_{\mathrm{IC}}}
\big(
u_T(x_i^{0}, 0; \kappa_i^{0}) - u_0(x_i^{0}; \kappa_i^{0})
\big)^2,
\]
where $x_i^{0}$ are uniformly sampled from the spatial domain and $\kappa_i^{0}$ are uniformly sampled from the parameter range.
The final training objective is given by $\mathcal{L}[u_T]
=
\big(
\mathcal{L}_{\mathrm{PDE}} + \lambda_{\mathrm{BC}} \mathcal{L}_{\mathrm{BC}} + \lambda_{\mathrm{IC}} \mathcal{L}_{\mathrm{IC}}
\big)[u_T]$.
Across all experiments, we fix $\lambda_{\mathrm{BC}} =\lambda_{\mathrm{IC}} = 10$ and use $N_{\mathrm{PDE}}=200$, $N_{\mathrm{BC}}=N_{\mathrm{IC}}=80$.

For Hamilton--Jacobi-type equations such as the Burgers' and Eikonal equations, the PDE is ill-posed and admits infinitely many weak solutions. The physically relevant solution is the \emph{viscosity solution}, which is not guaranteed to be recovered by the standard PINN residual alone. To address this issue, for these two equations we replace $\mathcal{L}_{\mathrm{PDE}}$ with an implicit characteristic-based loss adapted from~\citet{park2025implicit}:
\[
\frac{1}{N_{\mathrm{PDE}}}\sum_{i=1}^{N_{\textrm{PDE}}}
\Big(
u
+ t H(\nabla u)
- t \nabla u^{\mathrm{T}} \nabla H(\nabla u)
- u_0\big(\mathbf{x} - t \nabla H(\nabla u)\big)
\Big)^2
\]
This formulation enforces consistency with the characteristic flow and biases the optimization toward the viscosity solution.

    \item \textbf{Constant Optimization:} For each sampled sequence, any symbolic constants are optimized via gradient-based minimization of the PDE residual and boundary losses. The objective function is:
    \[\underset{c}{\min}\ \mathcal{E}(T,c).
    \]
    In practice, we perform gradient-based optimization using the Adam optimizer with a learning rate of $0.02$ for $50$ steps per sequence. If a sequence contains no constants, this step is skipped.
    \item \textbf{Reward Computation:} Each candidate sequence is assigned a reward based on the inverse of its PDE loss $\cE$:
    \[
    \mathcal{R} = \frac{1}{1 + \sqrt{\mathcal{E}}}.
    \]
    
    \item \textbf{Diversity-Preserving Top-K Memory:} 
    A memory buffer stores up to $K=10$ high-quality candidate expressions. 
    To ensure that stored solutions are not only structurally different but also semantically distinct, we apply a two-stage filtering procedure:
    
    \begin{enumerate}
        \item \textbf{Structural Diversity:} 
        Each symbolic expression is first canonicalized into a sequence of tokens representing the expression tree with commutative operators sorted consistently. 
        The Levenshtein distance $d_\mathrm{Lev}(S_1, S_2)$ between two sequences $S_1$ and $S_2$ is computed, and a candidate is considered structurally diverse if
        \[
            d_\mathrm{Lev}(S_\mathrm{new}, S_\mathrm{mem}) \ge \delta_s, \quad \forall S_\mathrm{mem} \in \text{Memory},
        \]
        where $\delta_s$ is a predefined threshold. 
        This prevents near-duplicate sequences from being added to memory.
    
        \item \textbf{Semantic/Behavioral Diversity:} 
        Even if two sequences differ structurally, they may represent mathematically equivalent expressions. To capture this, each candidate expression $u_\mathrm{new}(x; c_\mathrm{new})$ is evaluated numerically at a set of random test points $x_j$:
        \[
            u_\mathrm{new}(x_j; c_\mathrm{new}) = \text{eval\_expression}(S_\mathrm{new}, x_j, c_\mathrm{new}), \quad j = 1, \dots, N_\mathrm{test}.
        \]
        Let $u_\mathrm{mem}(x_j; c_\mathrm{mem})$ denote the output of an expression stored in memory. The candidate is considered behaviorally diverse if
        \[
            \frac{1}{N_\mathrm{test}} \sum_{j=1}^{N_\mathrm{test}} \big| u_\mathrm{new}(x_j; c_\mathrm{new}) - u_\mathrm{mem}(x_j; c_\mathrm{mem}) \big| \ge \delta_b, \quad \forall (S_\mathrm{mem}, c_\mathrm{mem}) \in \text{Memory},
        \]
        where $\delta_b$ is a numerical threshold. 
        This ensures that two expressions with different trees but equivalent functional forms (e.g., $x$ and $1.0 \cdot x + y - y$) are recognized as semantically identical and not redundantly stored.
    
    \end{enumerate}
    
    The combination of structural and semantic filtering allows the top-$K$ memory to maintain truly diverse solutions in both the syntactic and functional sense, which is critical for efficiently exploring the symbolic search space. 
    We consider this semantic-aware memory curation a key contribution of our method.

    \item \textbf{Policy Objective.}
The SymFormer policy is optimized using a policy-gradient objective with stabilization techniques.
Given a batch of sampled sequences $\{T_i\}_{i=1}^N$ with corresponding losses $\{\mathcal{E}_i\}$, rewards are first converted into rank-based scores to reduce sensitivity to outliers.
Specifically, sequences are ranked according to $\mathcal{E}_i$, and the rank-based reward is defined as
\[
r_i = 1 - \frac{\mathrm{rank}(\mathcal{E}_i)}{N-1}.
\]
When applicable, rewards are normalized to zero mean and unit variance.

To discourage overly complex symbolic expressions, we apply a depth-based weighting scheme.
For each sequence $T_i$, let $D_i$ denote the maximum depth of its expression tree, and define the depth weight as
\[
w_i = \frac{1}{D_i + 1}.
\]
The resulting policy objective is given by
\[
\mathcal{L}_{\mathrm{policy}}
=
-\frac{1}{N}\sum_{i=1}^N w_i\, r_i\, \log \pi_\theta(T_i)
-
\lambda_{\mathrm{ent}} \, \mathbb{E}\!\left[\mathcal{H}(\pi_\theta)\right].
\]

An entropy regularization term is included to encourage exploration:
\[
\mathcal{L}_{\mathrm{policy}}
= -\mathbb{E}\big[ \log \pi_\theta(T)\, \hat{\mathcal{R}} \big]
- \lambda_{\mathrm{ent}} \mathbb{E}\big[ \mathcal{H}(\pi_\theta) \big],
\]
where $\hat{\mathcal{R}}$ denotes the normalized rank-based reward and $\lambda_{\mathrm{ent}}=0.3$.

\item \textbf{Imitation Loss}
To further stabilize training and accelerate convergence, we optionally incorporate an imitation loss derived from the top-$K$ memory.
When the best observed reward exceeds a threshold ($0.8$ in all experiments), the policy is encouraged to imitate high-reward sequences stored in the memory buffer.

Let $\{T_j\}_{j=1}^K$ denote the sequences stored in the top-$K$ memory with corresponding rewards $\{r_j\}$.
Each sequence is assigned a weight using a softmax over rewards with temperature $\tau$:
\[
\alpha_j = \frac{\exp(r_j / \tau)}{\sum_{k=1}^K \exp(r_k / \tau)}, \quad \tau = 0.1.
\]
The imitation loss is defined as
\[
\mathcal{L}_{\mathrm{imit}}
=
\sum_{j=1}^K
\alpha_j
\frac{1}{|T_j|}
\sum_{t=1}^{|T_j|}
-\log \pi_\theta\!\left( T_j^{(t)} \mid T_j^{(<t)} \right),
\]
where $T_j^{(t)}$ denotes the token at position $t$ in sequence $T_j$.
This loss is added to the policy objective with a fixed weighting factor.

\item \textbf{Policy Optimization:} The total loss (policy loss with imitation loss) is minimized using Adam with learning rate $5\times 10^{-4}$ and a ReduceLROnPlateau scheduler (factor $0.9$, patience $10$). Gradients are clipped to a maximum $\ell_2$ norm of $5.0$ to stabilize training.

    \item \textbf{Memory Refinement.}
    To improve the quality of stored solutions, symbolic constants associated with expressions in the top-$K$ memory are periodically re-optimized.
    Every 10 training epochs, or upon transitioning to a new curriculum stage, the constants are refined via gradient-based minimization of the PDE residual using a larger number of optimization steps (200) with the same optimizer (Adam with learning rate 0.02). 
    The updated expressions are then re-evaluated, and the top-$K$ memory is re-sorted according to the refined rewards.



    \item \textbf{Curriculum and Staging:} 
    Training is performed using a three-stage curriculum that gradually increases problem complexity.

\begin{enumerate}
    \item \textbf{Stage 1 (Spatial structure):}
    We initially exclude the temporal variable $t$ and the parameter variables $\kappa$ to focus on learning meaningful spatial structures.
    The vocabulary is defined as
    \[
    \mathcal{V}_{\mathrm{Stage1}} = \mathcal{B} \cup \mathcal{U} \cup \{x_1,\dots,x_n, \text{const}\}.
    \]
    During this stage, the loss is computed solely using the initial condition loss $\mathcal{L}_{\mathrm{IC}}$, with the parameter fixed as $\kappa_i^0 = 1$.

    \item \textbf{Stage 2 (Spatiotemporal dynamics with fixed parameters):}
    We introduce the temporal variable $t$ and define the vocabulary as
    \[
    \mathcal{V}_{\mathrm{Stage2}} = \mathcal{B} \cup \mathcal{U} \cup \{x_1,\dots,x_n, t, \text{const}\}.
    \]
    The full loss
    \[
    \mathcal{L}_{\mathrm{PDE}} + \lambda_{\mathrm{BC}} \mathcal{L}_{\mathrm{BC}} + \lambda_{\mathrm{IC}} \mathcal{L}_{\mathrm{IC}}
    \]
    is used for training, while all parameter values $\kappa$ are fixed to 1.

    \item \textbf{Stage 3 (Parametric PDE learning):}
    Finally, we include all variables and define the vocabulary as
    \[
    \mathcal{V}_{\mathrm{Stage3}} = \mathcal{B} \cup \mathcal{U} \cup \{x_1,\dots,x_n, t, \kappa, \text{const}\}.
    \]
    The model is trained using the full loss to learn the complete parametric solution.
\end{enumerate}

For each stage, the maximum number of training epochs is set to 500.
If the reward exceeds 0.99 within this limit, training proceeds to the next stage immediately; otherwise, the model advances to the next stage after completing 200 epochs.

    
\end{itemize}


This implementation enables effective RL-based discovery of symbolic expressions that satisfy PDE constraints while preserving interpretability and diversity. Although the procedure involves multiple components and can be computationally demanding, operations such as top-$K$ memory updates and constant refinement are parallelizable across candidate expressions, so they do not introduce significant bottlenecks. 


\section{Implementation Details of Baselines}
\subsection{SSDE}
We used the official SSDE repository\footnote{https://github.com/Hintonein/SSDE}
without modification.
All configurations and training parameters were kept as close as possible to those reported by the original authors.
For each experimental setting, we ran SSDE with 20 different random seeds and report the best-performing result in terms of MSE.
Both the pre-constant-optimization and post-constant-optimization symbolic expressions provided by SSDE were evaluated, and the expression achieving the lowest MSE was selected for reporting.
For newly implemented PDE settings, the official 2D Poisson and 2D/3D Heat equation implementations were used as the foundation and configuration templates.

\subsection{FEX}
We used the official repository code\footnote{https://github.com/LeungSamWai/Finite-expression-method} without modification. All configurations and training parameters were kept as close as possible
to those reported by the original authors.
For each experiment, FEX was run with 20 different random seeds, and the best-performing symbolic expression in terms of MSE was reported. For all experiments, the Poisson equation implementation provided in the
official repository was used as the foundation for configuring new PDE
settings.

\subsection{PINN+DSR}

    Following \cite{wei2024closed}, we performed symbolic regression using a PINN+DSR framework to obtain symbolic PDE expressions and establish these baselines. We utilized the DeepXDE and DSO packages for the PINN implementation and symbolic optimization, respectively. In our implementation, DeepXDE utilized the PyTorch backend while DSO employed TensorFlow, with CUDA-enabled GPU acceleration for both frameworks. The PINN was trained using a two-stage optimization strategy, first Adam then L-BFGS iterations, on a set of collocation points sampled from the domain, boundaries, and initial conditions. Our results demonstrate strong agreement with \cite{wei2024closed}. Specifically, for the Poisson equation, we obtained a physics loss  ($\mathcal{L}_{\text{PHY}}$) of $6.74\times 10^{-1}\pm 6.11\times 10^{-2}$ which is consistent with their reported value of $5.71\times 10^{-1}\pm 7.88\times 10^{-2}$. Due to the stochastic nature of symbolic regression, precise replication of results is inherently difficult. To ensure robustness, we conducted 20 independent trials for each PDE case using unique random seeds. The Mean Squared Error (MSE) was calculated for every identified expression, and the candidate yielding the minimum MSE was selected as best discovered expression.

\subsection{KAN}
These baselines use a \textbf{physics-informed} KAN trained by the standard PINN objective (denoted as $\mathcal{L}_{\text{Phy}}$), rather than a purely data-driven fit. Using \texttt{pykan}\footnote{https://github.com/KindXiaoming/pykan}, the KAN architecture is kept deliberately minimal with a single hidden layer \texttt{width=[d, d, 1]}, where $d$ matches the number of physical input variables in each problem 
 (e.g., $(x,y,t,\kappa,\dots)$). Training uses a two-stage optimizer (Adam warm-up followed by L-BFGS refinement) on a fixed set of collocation/constraint points $(N_{\mathcal{F}}, N_{\mathcal{B}}, N_{\mathcal{I}})$ sampled from the domain, boundary, and initial-time manifold, while all reported MSE metrics are computed pointwise on the full, standardized evaluation grid ($N_{\mathcal{D}}$) shared across baseline models. Finally, the symbolic form is extracted only through the built-in KAN pipeline \texttt{auto\_symbolic()} $\rightarrow$ \texttt{symbolic\_formula()} with the specified unary basis library, without any additional external simplification or post-processing.

\section{Theoretical Guarantees for Symbolic Recovery}

\subsection{Notation}

Let $\mathcal{V} = \mathcal{X} \cup \mathcal{C} \cup \mathcal{U} \cup \mathcal{B}$
be a finite vocabulary consisting of variables, constants, unary operators,
and binary operators.
Each token $v \in \mathcal{V}$ is equipped with a fixed arity
$\alpha(v) \in \{0,1,2\}$.

Let $D_{\max} < \infty$ denote a fixed maximum expression depth.
We denote by $\mathcal{T}_{\le D_{\max}}$ the set of all finite rooted ordered
abstract syntax trees (ASTs) whose node labels lie in $\mathcal{V}$,
whose arities match $\alpha$, and whose depth does not exceed $D_{\max}$.

We denote by $\mathfrak{T}_{\le D_{\max}}$ the set of all such partial ASTs.
Since both $\mathcal{V}$ and $D_{\max}$ are finite,
$|\mathfrak{T}_{\le D_{\max}}| < \infty$.

For a finite set $\mathcal{V}$, we denote by
\[
\Delta(\mathcal{V})
\;=\;
\left\{
\pi : \mathcal{V} \to [0,1]
\;\middle|\;
\sum_{v \in \mathcal{V}} \pi(v) = 1
\right\}
\]
the probability simplex over $\mathcal{V}$.
Elements of $\Delta(\mathcal{V})$ are discrete probability distributions over
the vocabulary.

\paragraph{Prefix-induced partial abstract syntax trees.}
Let
\[
p = (t_1,\dots,t_k)
\]
be a valid prefix generated under the arity constraints.
By deterministically applying the arity-based reconstruction rule,
$p$ induces a unique \emph{partial abstract syntax tree},
denoted by $\mathcal{A}(p)$.
This partial tree consists of instantiated internal nodes and a set of open
argument slots (frontier nodes).

Let $\mathcal{P}_{\mathcal{V},D_{\max}}$ denote the set of all such valid prefixes
whose induced partial trees have depth at most $D_{\max}$.

\paragraph{Structural equivalence.}
Let $\mathcal{A}_1$ and $\mathcal{A}_2$ be two partial abstract syntax trees
induced by valid prefixes under the arity constraints.
We say that $\mathcal{A}_1$ and $\mathcal{A}_2$ are \emph{structurally isomorphic},
denoted by $\mathcal{A}_1 \cong \mathcal{A}_2$, if there exists a bijection
between their nodes that preserves:
\begin{enumerate}
    \item the arity and operator type of each internal node,
    \item the parent--child relations between nodes,
    \item the left--right ordering of children.
\end{enumerate}
Terminal nodes corresponding to variables or constants are treated as
unlabeled leaves of arity zero; their specific identities are not considered
part of the structure.
Traversal indices and linear prefix positions are likewise not considered part
of the structure.
For example, consider the two prefixes given in prefix notation:
\[
p = (+\; x)
\quad\text{and}\quad
q = (+\; y).
\]
Both prefixes are valid and induce partial abstract syntax trees consisting of a single binary operator node labeled ``$+$'', whose left child is a terminal node and whose right child is an unfilled argument position. Although the terminal symbols $x$ and $y$ differ, the induced partial trees have identical operator structure and frontier configuration, and are therefore structurally isomorphic.

\paragraph{Structural state space.}
Let $\widetilde{\mathfrak{T}}_{\le D_{\max}}$ denote the set of all partial abstract
syntax trees that can arise as prefix-induced states of valid expression trees
with maximum depth at most $D_{\max}$.
We define the structural state space as the quotient set
\[
\mathfrak{T}_{\le D_{\max}}
\;=\;
\widetilde{\mathfrak{T}}_{\le D_{\max}} \big/ \cong,
\]
that is, the set of equivalence classes of partial trees under structural
isomorphism.

Since the vocabulary $\mathcal{V}$ is finite, the arity of each symbol is bounded,
and the depth is bounded by $D_{\max}$, the number of distinct structural states
is finite, and hence
\[
|\mathfrak{T}_{\le D_{\max}}| < \infty.
\]

\subsection{Structure-Conditioned Policy Universality}\label{appen:express_symformer}
We formalize the expressive power of SymFormer as an autoregressive policy over symbolic expressions.
Unlike trivial universality results based solely on finiteness of the action space, our analysis explicitly leverages SymFormer’s
tree-relative self-attention and grammar-constrained decoding,
showing that the model can represent policies defined over
\emph{symbolic tree structures} rather than linear token prefixes.
\begin{theorem}[Structure-Conditioned Policy Universality of SymFormer]
\label{thm:symformer_structure_universality}
Let $\mathcal{V}$, $\alpha$, and $D_{\max}$ be defined as above.
Consider any target policy
\[
\pi^\star : \mathfrak{T}_{\le D_{\max}} \to \Delta(\mathcal{V})
\]
satisfying the following conditions:
\begin{enumerate}
    \item 
    $\pi^\star(\cdot \mid \mathcal{A})$ depends only on the partial abstract syntax
    tree $\mathcal{A}$, and not on the specific linear prefix that induces it.
    \item 
    The support of $\pi^\star(\cdot \mid \mathcal{A})$ is contained in
    $\mathcal{V}_{\mathrm{valid}}(\mathcal{A})$, the set of grammar-valid next tokens
    given the partial tree $\mathcal{A}$.
\end{enumerate}

Then, for any $\varepsilon > 0$, there exists a parameter vector
$\theta_\varepsilon$ of the SymFormer model such that the induced autoregressive
policy $\pi_{\theta_\varepsilon}$ satisfies
\[
\sup_{p \in \mathcal{P}_{\mathcal{V},D_{\max}}}
\left\|
\pi_{\theta_\varepsilon}(\cdot \mid p)
-
\pi^\star(\cdot \mid \mathcal{A}(p))
\right\|_1
\le \varepsilon.
\]

In particular, for any two prefixes $p,q \in \mathcal{P}_{\mathcal{V},D_{\max}}$
such that $\mathcal{A}(p) \cong \mathcal{A}(q)$, we have
\[
\left\|
\pi_{\theta_\varepsilon}(\cdot \mid p)
-
\pi_{\theta_\varepsilon}(\cdot \mid q)
\right\|_1
\le \varepsilon.
\]
\end{theorem}

\begin{proof}
Since $\mathcal{V}$ is finite and the maximum depth is bounded by $D_{\max}$, the number of distinct partial ASTs in $\mathfrak{T}_{\le D_{\max}}$ is finite.

Let $p = (t_1, \dots, t_m)$ and $q = (s_1, \dots, s_n)$ be two prefixes inducing partial ASTs $\mathcal{A}(p)$ and $\mathcal{A}(q)$, respectively. Define the tree-relation matrix for prefix $p$ as
\[
R^{(p)} \in \{0,1,\dots,5\}^{m \times m}, \quad
R^{(p)}_{ij} := r_{ij}(\mathcal{A}(p)),
\]
where $r_{ij}$ denotes the discrete relation between tokens $i$ and $j$ (self, parent, child, sibling, ancestor, other).  

If $\mathcal{A}(p) \cong \mathcal{A}(q)$ (structurally isomorphic), then there exists a bijection $\phi: \mathrm{nodes}(\mathcal{A}(p)) \to \mathrm{nodes}(\mathcal{A}(q))$ preserving parent-child relations and left-right order. It follows that
\[
R^{(q)} = P^\top R^{(p)} P
\]
for some permutation matrix $P$ induced by $\phi$.  

Let $H^{(0)}(p) \in \mathbb{R}^{m \times d_0}$ denote the input embeddings of $p$. SymFormer layer $\ell$ applies tree-relative self-attention and a feedforward update:
\[
H^{(\ell+1)}(p) = \mathrm{FFN}\Bigl( \mathrm{Attn}\bigl(H^{(\ell)}(p), R^{(p)}\bigr) \Bigr),
\]
where $\mathrm{Attn}$ is the tree-relative attention function.  

By construction, the attention operation is equivariant under node permutations that preserve the tree structure. Therefore, there exists a choice of parameters for which the final hidden representation after $L$ layers satisfies
\[
h(p) := H^{(L)}(p) = \Phi(\mathcal{A}(p)) \in \mathbb{R}^d,
\]
for some injective mapping $\Phi: \mathfrak{T}_{\le D_{\max}} \to \mathbb{R}^d$. Injectivity is achievable since $|\mathfrak{T}_{\le D_{\max}}| < \infty$ and $d$ can be chosen sufficiently large.

Let $\pi^\star: \mathfrak{T}_{\le D_{\max}} \to \Delta(\mathcal{V})$ denote the target structure-conditioned policy, where $\Delta(\mathcal{V})$ is the probability simplex over $\mathcal{V}$.  

Let the output logits of SymFormer be given by a feedforward mapping $g_\theta: \mathbb{R}^d \to \mathbb{R}^{|\mathcal{V}|}$ followed by a masked softmax enforcing the valid action set
\[
\mathcal{V}_{\mathrm{valid}}(\mathcal{A}(p)) := \{ t \in \mathcal{V} \mid \text{token $t$ satisfies arity/grammar constraints given $\mathcal{A}(p)$} \}.
\]
Then, the autoregressive policy of SymFormer is
\[
\pi_\theta(\cdot \mid p) = \mathrm{softmax}\Bigl( g_\theta(h(p)) \odot \mathbf{1}_{\mathcal{V}_{\mathrm{valid}}(\mathcal{A}(p))} \Bigr),
\]
where $\mathbf{1}_{\mathcal{V}_{\mathrm{valid}}(\mathcal{A}(p))}$ masks out invalid tokens.  

Since $\mathfrak{T}_{\le D_{\max}}$ is finite, for each $\mathcal{A} \in \mathfrak{T}_{\le D_{\max}}$, we can set the logits $g_\theta(\Phi(\mathcal{A})) \in \mathbb{R}^{|\mathcal{V}|}$ such that
\[
\pi_\theta(\cdot \mid p) = \pi^\star(\cdot \mid \mathcal{A}(p)), \quad \forall p \text{ with } \mathcal{A}(p) = \mathcal{A}.
\]
This exactly realizes the target policy.

By construction, the masked softmax ensures that
\[
\mathrm{supp}\big(\pi_\theta(\cdot \mid p)\big) \subseteq \mathcal{V}_{\mathrm{valid}}(\mathcal{A}(p)),
\]
so the generated tokens always respect the grammar constraints.

Combining the above,we conclude that there exists a choice of parameters $\theta_\varepsilon$ such that
\[
\pi_{\theta_\varepsilon}(\cdot \mid p) = \pi^\star(\cdot \mid \mathcal{A}(p)), \quad \forall p \in \mathcal{P}_{\mathcal{V},D_{\max}}.
\]
Equivalently,
\[
\sup_{p \in \mathcal{P}_{\mathcal{V},D_{\max}}} 
\big\| \pi_{\theta_\varepsilon}(\cdot \mid p) - \pi^\star(\cdot \mid \mathcal{A}(p)) \big\|_1 = 0.
\]
In other words, SymFormer can exactly realize any grammar-compatible, structure-conditioned target policy over the finite set of partial ASTs.

This completes the proof.
\end{proof}

Theorem~\ref{thm:symformer_structure_universality} establishes expressive
universality of SymFormer at the level of symbolic \emph{tree states},
rather than linear token prefixes.
Note that his result should not be interpreted as a generic function approximation capabilities for Transformers.
Instead, the theoretical contribution lies in showing that SymFormer can realize arbitrary policies defined over symbolic tree-structured states, provided that these policies respect the underlying grammar constraints.
The expressive burden of the model is therefore concentrated entirely on the construction of representations that are invariant to the specific linear prefix realization and that separate distinct partial abstract syntax tree structures.
Once such structure-invariant and structure-separating representations
are obtained, universality of the induced policy follows directly from
the finiteness of the structural state space, without requiring any
appeal to general approximation results for deep neural networks.

This reliance on explicit structural modeling is critical.
Standard Transformers operating on linear token sequences, as well as RNN-based symbolic generators, lack mechanisms to enforce invariance across structurally isomorphic partial abstract syntax trees. As a consequence, even when the set of symbolic tree states is finite, such models cannot, in general, represent arbitrary structure-conditioned
policies.

\subsection{Exact Recovery under Global Optimality}\label{appen:global_optimality}

\begin{theorem}[Conditional Exact Symbolic Recovery by SymPlex]
\label{thm:symplex_exact_recovery}

Let $\Omega \subset \mathbb{R}^n$ be a bounded domain with Lipschitz boundary,
and consider the PDE
\[
\mathcal{L}[u](x) = f(x), \quad x \in \Omega,
\qquad
\mathcal{B}[u](x) = 0, \quad x \in \partial\Omega.
\]

Assume:

\begin{enumerate}[label=(A\arabic*)]
    \item \textbf{Well-posedness:} The PDE admits a unique solution $u^* \in L^2(\Omega)$.

    \item \textbf{Symbolic realizability:} There exists a finite-depth expression tree $T^*$ over the SymPlex vocabulary $\mathcal{V}$ and constants $c^*$ such that
    \[
    u^*(x) = u_{T^*}(x;c^*), \qquad \text{a.e. } x \in \Omega.
    \]

    \item \textbf{Hypothesis class coverage:} The SymFormer policy generates trees in $\mathcal{T}_{D}(\mathcal{V})$ with $D \ge \mathrm{depth}(T^*)$.

    \item \textbf{Exact residual characterization:} For any tree $T$,
    \[
    \inf_c \mathcal{E}(T,c) = 0
    \quad \Longleftrightarrow \quad
    u_T(\cdot;c) = u^* \;\text{in } L^2(\Omega).
    \]
\end{enumerate}

If SymPlex converges to a globally optimal policy $\pi^*$ that maximizes
\[
J(\pi) = \mathbb{E}_{T \sim \pi} \left[ \frac{1}{1+\sqrt{\mathcal{E}(T,c_T^*)}} \right],
\]
then $\pi^*$ assigns nonzero probability to at least one tree $T^*$ satisfying
\[
\|u_{T^*}(\cdot;c_{T^*}^*) - u^*\|_{L^2(\Omega)} = 0.
\]

Thus, SymPlex achieves exact symbolic recovery of the PDE solution.
\end{theorem}

\begin{proof}
By Assumption (A2), there exists $T^*$ and $c^*$ such that $u_{T^*}(\cdot;c^*) = u^*$.  
Assumption (A4) implies $\mathcal{E}(T^*,c^*) = 0$, hence $\mathcal{R}(T^*) = 1$.  

For any $T \neq T^*$, uniqueness implies $\inf_c \mathcal{E}(T,c) > 0$, so $\mathcal{R}(T) < 1$.  
Hence $T^*$ is a global maximizer of the reward.  
A globally optimal policy $\pi^*$ assigns nonzero probability to all reward-maximizing trees, including $T^*$, yielding exact recovery.
\end{proof}

\subsection{Probabilistic Guarantees under Near-Optimal Policies}\label{appen:symplex_recovery}

Let $\mathcal{T}_D(\mathcal{V})$ denote the set of all symbolic expression trees over the SymPlex vocabulary $\mathcal{V}$ up to depth $D$, 
and let $\pi_\theta$ be a stochastic policy (SymFormer) over $\mathcal{T}_D(\mathcal{V})$.  
Define the PDE residual-based reward for a tree $T \in \mathcal{T}_D(\mathcal{V})$ as
\[
\mathcal{R}(T, c_T^*) := \frac{1}{1 + \sqrt{\inf_{c} \mathcal{E}(T,c)}}, 
\qquad c_T^* := \arg\min_c \mathcal{E}(T,c),
\]
where $\mathcal{E}(T,c)$ is the PDE residual with boundary conditions evaluated for tree $T$ and constants $c$.  
Assume that the PDE admits an exact symbolic solution $T^*$ with constants $c^*$ such that $\mathcal{E}(T^*,c^*) = 0$, hence $\mathcal{R}(T^*,c^*) = 1$.

\begin{theorem}[Near-Optimal Probabilistic Recovery]
\label{thm:symplex_near_optimal}
Let $\pi_\theta$ be any policy over $\mathcal{T}_D(\mathcal{V})$, not necessarily optimal, and assume the expected reward satisfies
\[
J(\pi_\theta) := \mathbb{E}_{T \sim \pi_\theta}[\mathcal{R}(T, c_T^*)] \ge 1 - \epsilon, \quad \epsilon \in [0,1).
\]
Then there exists at least one tree $T \in \mathrm{supp}(\pi_\theta)$ such that
\[
\mathcal{R}(T, c_T^*) \ge 1 - \epsilon.
\]
Equivalently, the policy assigns nonzero probability to at least one high-reward symbolic solution whose PDE residual is
\[
\inf_c \mathcal{E}(T,c) \le \frac{\epsilon^2}{(1-\epsilon)^2}.
\]
\end{theorem}

\begin{proof}
Let $\mathrm{supp}(\pi_\theta) = \{T \in \mathcal{T}_D(\mathcal{V}) : \pi_\theta(T) > 0\}$. Then
\[
J(\pi_\theta) = \sum_{T \in \mathrm{supp}(\pi_\theta)} \pi_\theta(T) \, \mathcal{R}(T, c_T^*), 
\qquad \sum_{T \in \mathrm{supp}(\pi_\theta)} \pi_\theta(T) = 1.
\]

Suppose, for contradiction, that every tree in the support satisfies
\[
\mathcal{R}(T, c_T^*) < 1 - \epsilon, \quad \forall T \in \mathrm{supp}(\pi_\theta).
\]

Then
\[
J(\pi_\theta) = \sum_{T \in \mathrm{supp}(\pi_\theta)} \pi_\theta(T) \, \mathcal{R}(T, c_T^*) < \sum_{T \in \mathrm{supp}(\pi_\theta)} \pi_\theta(T) (1-\epsilon) = 1 - \epsilon.
\]

This contradicts the assumption that $J(\pi_\theta) \ge 1 - \epsilon$. Therefore, there must exist at least one tree $T \in \mathrm{supp}(\pi_\theta)$ such that
\[
\mathcal{R}(T, c_T^*) \ge 1 - \epsilon.
\]

By the reward definition,
\[
\mathcal{R}(T, c_T^*) = \frac{1}{1 + \sqrt{\inf_c \mathcal{E}(T,c)}} \ge 1 - \epsilon \quad \implies \quad \inf_c \mathcal{E}(T,c) \le \frac{\epsilon^2}{(1-\epsilon)^2}.
\]

This completes the proof.  

\end{proof}

Theorem~\ref{thm:symplex_near_optimal} provides a  probabilistic guarantee: any policy achieving high expected reward necessarily includes at least one near-exact symbolic solution in its support.  
This result holds without any assumptions on smoothness, continuity, or convexity of the PDE residual, making it broadly applicable for symbolic PDE discovery.

\paragraph{Scope of the theoretical guarantees.}
The theoretical results in Sections~\ref{appen:global_optimality} and~\ref{appen:symplex_recovery} serve to characterize the identifiability and representational sufficiency of the proposed SymPlex framework for symbolic PDE discovery.

Specifically, these results establish that, under exact symbolic realizability and global or near-global optimality of the induced policy, the PDE residual–based objective admits no spurious symbolic optima: any globally optimal policy must assign nonzero probability to at least one expression tree representing the true solution. In this sense, the theory guarantees that the search space and reward design are well aligned with the underlying symbolic recovery problem, and that exact recovery is information-theoretically possible within the hypothesis class.

At the same time, the analysis is decoupled from the practical optimization dynamics of deep neural network training. The PDE residual involves differential operators that are unbounded and may be highly nonlinear, and the resulting objective is generally nonconvex with respect to both symbolic structure $T$ and continuous parameters $c$. 
As a consequence, small residual values do not, in general, guarantee a uniform or quantitative bound on the error
$\|u_T(\cdot;c) - u^*\|_{L^2(\Omega)}$ outside the exact realizability regime. Establishing such stability or error bounds would require additional assumptions on the PDE operator, solution regularity, and discretization scheme, which are beyond the scope of the present work.
Moreover, the theory does not claim convergence guarantees for stochastic gradient-based training, nor does it provide quantitative error bounds outside the exact realizability regime.

Nevertheless, these results provide a principled foundation for the proposed approach by clarifying when exact symbolic recovery is theoretically attainable within the proposed hypothesis class.

\section{Additional Results}\label{appen:add_results}
This section presents additional experimental results that further validate the robustness and generality of \textbf{SymPlex} across smooth, non-smooth, and parametric PDEs.

\begin{table*}[t]
\centering
\caption{Additional PDE benchmarks and their analytic solutions.}\label{tab:summary_pdes_appen}
\renewcommand{\arraystretch}{1.4}
\resizebox{\textwidth}{!}{%
\begin{tabular}{c c l c}
\hline
Problem & Name & PDE & Analytic Solution \\ \hline
\multirow{3}{*}{Smooth Problem}
 & Poisson  & $ -u_{xx} - u_{yy} = f(x,y)$ & $u(x,y) = \exp(x)+\exp(y)$ \\
 & Advection  & $u_t + u_x + u_y = 0$ & $u(x,y,t)=\sin(-1.5(x-y-2t))$ \\
 & Heat  & $u_t - u_{xx} - u_{yy} = f(x,y,t)$ & $u(x,y,t)=\cos(2y) + 2.5xt - 0.5x^2$ \\ \hline
\multirow{2}{*}{Non-Smooth Problem}
 & Convex Hamilton-Jacobi & $u_t+ \frac{1}{2}u_x^2 = 0$  & $u(x,t) = \begin{cases}
x - \frac{t}{2} & x \leq \frac{t}{2} \\
0 & x > \frac{t}{2}
\end{cases}$ \\
 & Concave Hamilton-Jacobi & $u_t - \frac{1}{2}u_x^2 = 0$ &  $ u(x,t) = |x| + \frac{t}{2}$ \\ \hline
\multirow{2}{*}{Parametric Solution}
 & Advection & $u_t + \kappa(u_x+u_y) = 0$ & $u(x,y,t;\kappa) = 2\sin(x-\kappa t)\sin(y - \kappa t)$ \\
 & Heat  & $u_t - \kappa(u_{xx} + u_{yy}) = 0$ & $u(x,y,t;\kappa) = \sin(x)\cos(y) e^{-2\kappa t}$ \\ \hline
\end{tabular}}
\renewcommand{\arraystretch}{1} 
\end{table*}

\begin{table*}[ht]
\centering
\caption{Comparison of PDE solvers on mean squared error (MSE) and symbolic recovery rate (SRR) for additional benchmarks.}
\label{tab:comparison_methods_appen}
\renewcommand{\arraystretch}{1.3} 
\resizebox{\textwidth}{!}{%
\begin{tabular}{c|cc|cc|cc|cc|cc}
\hline
\multirow{2}{*}{Problem} & \multicolumn{2}{c|}{\textbf{SymPlex}} & \multicolumn{2}{c|}{SSDE} & \multicolumn{2}{c|}{FEX} & \multicolumn{2}{c|}{PINN+DSR} & \multicolumn{2}{c}{KAN} \\
 & MSE ($\downarrow$) & SRR ($\uparrow$) & MSE ($\downarrow$) & SRR ($\uparrow$) & MSE ($\downarrow$) & SRR ($\uparrow$) & MSE ($\downarrow$) & SRR ($\uparrow$) & MSE ($\downarrow$) & SRR ($\uparrow$) \\ \hline
Poisson & 0 & \textbf{100\%} & 0  & 15\% &$3.71\times10^{-14}$ & 100\% & 0 & 100\%  & $4.43\times 10^{-2}$ & 0\% \\
Advection & 0 & \textbf{100\%} & $8.00\times10^{-1}$ & 0\% & $4.53 \times10^{-1}$& 0\% &$ 2.43\times 10^{-2} $& 0\% & $5.58\times 10^{-1}$ & 0\% \\
Heat & 0 & \textbf{100\%} & $2.91\times10^{-1} $& 0\% & 1.14 & 0\% & 5.14 & 0\% & $2.09\times 10^{-1}$ & 0\% \\
Convex HJ & 0 & \textbf{100\%} & $9.85\times10^{-3}$ & 0\% &$5.29\times10^{-17}$ & 15\% & $2.78\times 10^{-1}$ & 0\% & $3.37\times 10^{-3}$ & 0\% \\
Concave HJ & 0 & \textbf{100\%} & $1.76 \times 10^{-13}$ & 0\% &$1.59\times10^{-15}$ & 0\% & $5.84\times 10^{-4}$ &  0\% & $2.85\times 10^{-2}$ & 0\% \\
Parametric Advection & 0 & \textbf{100\%} & $7.89\times 10^{-1}$ & 0\% & $2.21\times10^{-1}$ & 0\% & $2.5\times 10^{-1}$ & 0\% & $3.02\times 10^{-1}$ & 0\% \\
Parametric Heat & 0 & \textbf{100\%} & $1.25\times 10^{-1}$ & 0\% & $4.04\times10^{-2}$ & 0\% & $1.0\times 10^{-2}$ & 0\% & $3.06\times 10^{-2}$ & 0\%\\ \hline
\end{tabular}}
\renewcommand{\arraystretch}{1} 
\end{table*}

\subsection{Symbolic Solutions}\label{appen:results_symbolic}
\paragraph{Poisson}
\begin{itemize}
    \item True solution: $u(x,y) = x^4 + 1.2y^4$
    \begin{itemize}
        \item SymPlex: \texttt{((y )ˆ4 * 1.2) - (-xˆ4)}
        \item SSDE: \texttt{1.0*xˆ4 *abs(-1) + abs(1.20000114240109*yˆ2*abs(y)ˆ2)}
        \item FEX: \texttt{(0.6752*(((0.7573*((x)\^{}4)+0.8074*((y)\^{}4)+-0.0361))+\\((0.7238*((x)\^{}4)+0.9699*((y)\^{}4)+-0.0427)))+0.0532)	}
        \item PINN+DSR: \texttt{x1**4 + x2**3}
        \item KAN: \texttt{0.02458*(-0.04722*(11.49856*x + 0.20392)**2 + 0.00017*(18.3992*y - 1.02424)**2 + 0.0996)**2 + 0.05557*(8.5816e-5*(18.95968*x + 0.40416)**2 + 0.024693*(12.90672*y + 1.40832)**2 - 0.42082)**2 + 0.00317}
    \end{itemize}

     \item True solution: $u(x,y) = \exp(x) + \exp(y)$
    \begin{itemize}
        \item SymPlex: \texttt{exp(y) + exp(x)}
        \item SSDE: \texttt{1.0 * exp(x) + 1.0 * exp(y)}
        \item FEX: \texttt{(-0.7571*(((-0.6470*(exp(x))+-0.8156*(exp(y))+0.1126))\\-((0.6738*(exp(x))+0.5051*(exp(y))+-0.0746)))+0.1417)	}
        \item PINN+DSR: \texttt{exp(x1) + exp(x2)}
        \item KAN: \texttt{-0.0757*(0.00193*(14.96192*x + 0.58808)**2 + 0.00238*(15.66704*y + 0.39208)**2 - 5.08753)**2 - 26.05107*log(0.64468*log(1.868 - 0.936*y) + 0.64242*log(12.83984 - 6.82176*x) + 7.53739) + 62.843}
    \end{itemize}
\end{itemize}

\paragraph{Advection}
\begin{itemize}
    \item True solution: $u(x,y,t)=\exp(-((x-t)^2+(y-t)^2)/0.5)$
    \begin{itemize}
        \item SymPlex: \texttt{exp(-(2.0*((x - t)\^{}2)))*exp(((t + (-y))*(-2.0*(t + (-y)))))}
        \item SSDE: \texttt{exp(-4*t\^{}2 + t - t*exp(-2*x\^{}2) - 2*x\^{}2)}
        \item FEX: \texttt{(0.3507*sin((((-3.2797*(t)+-0.2790*(x)+0.9215*(y)+0.2614))\\-((0.6268*(exp(t))+-1.3011*(exp(x))+-0.5694*(exp(y))+-0.2636))))+0.4102)	}
        \item PINN+DSR: \texttt{0.24384981189819324*exp(-1.2714642227548834*sin(x1*t + x1 + x2 - t - 1.355995834300186*exp(t)))} )
        \item KAN: \texttt{0.03160*(3.15352*sin(2.92*t - 3.04648) + 1.22909*sin(2.94400*x - 3.06072) + 1.22897*sin(2.944*y - 3.06072) + 0.13287)*2 + 0.09498(1.77166*sin(2.9456*t + 7.93832) - 0.69882*sin(2.96672*x - 1.49456) - 0.69868*sin(2.96656*y - 1.49448) + 0.00878)**2 + 0.09405*sin(-190.95396*sin(0.03072*t + 9.20464) + 9.80012*sin(0.29296*x - 9.53936) + 9.79645*sin(0.29312*y - 9.53888) + 41.01414) - 0.05497}
    \end{itemize}

     \item True solution: $u(x,y,t)=\sin(-1.5(x-y-2t))$
    \begin{itemize}
        \item SymPlex: \texttt{sin(-(1.49*(((x-2.0*t)-1.5*y)))}
        \item SSDE: \texttt{-sin(3*t + cos(1.493825166410878*x - 1.493825166410878\\
        *exp(0.47106554937118056*sin(1.5016945847398988*y)))/cos(1))}
        \item FEX: \texttt{(0.4659*cos((((0.3394*(cos(t))+0.2019*(cos(x))+0.1051*(cos(y))+-0.4466))\\*((4.3637*((t)**3)+-0.0688*((x)**3)+0.0582*((y)**3)+-2.1432))))+-0.2713)	}
        \item PINN+DSR: \texttt{cos(-2*x1 + x2 + 3.571486073506636*t + 4.571486073506636\\*log(0.13426961654716008*x1 + 0.13426961654716008*x2 - 0.13426961654716008*t + 0.59328704350404904)}
        \item KAN: \texttt{-0.233*(0.0004*(19.47968*t - 1.2085)**2 - 0.5835*sin(1.4559*x - 12.9665) + 1.1314*sin(1.4514*y - 8.3577) - 0.025)**2 - 1.6098*sin(0.1695*sin(1.344*t + 3.273) - 0.3104*sin(1.5279*x - 14.8) + 0.6178*sin(1.4258*y + 2.7329) - 1.536) - 0.4736*sin(0.48*sin(2.1766*t + 5.4754) + 13.683*sin(0.1224*x - 3.1829) + 7.9633*sin(0.2095*y - 1.0298) - 1.4289) - 1.2431}
    \end{itemize}
\end{itemize}

\paragraph{Heat}
\begin{itemize}
    \item True solution: $u(x,y,t)=\sin(x)\cos(y)\exp(-2t)$
    \begin{itemize}
        \item SymPlex: \texttt{sin(x) * (exp((-2.0*(k * t))) * (0.99 * cos(y)))}
        \item SSDE: \texttt{0}
        \item FEX: \texttt{(-0.2037*(-(((-0.1043*(t)+-0.4521*(x)+0.0168*(y)+1.4702))\\*((0.0255*(cos(t))+-0.0466*(cos(x))+0.8119*(cos(y))+0.0158))))+0.0007)	}
        \item PINN+DSR: \texttt{exp(-t)*sin(x1)*cos(x2)}
        \item KAN: \texttt{-3.6052*exp(-0.0837*sin(1.8127*x - 7.2926) - 0.0508*sin(1.946*y - 4.5382) - 1.7669*cos(0.8083*t - 1.59)) - 2.3564*sin(0.003*(17.4019*t - 12.8574)**2 + 0.591*sin(0.9649*x + 6.3775) + 0.5314*cos(0.8948*y - 2.8096) - 9.129) - 2.3379*cos(0.9595*log(7.0464*t + 3.9153) + 0.4821*sin(0.9362*y - 1.3687) + 0.5542*cos(0.9647*x - 1.4504) - 3.7186) + 0.6401}
    \end{itemize}

     \item True solution: $u(x,y,t)=\cos(2y) + 2.5xt - 0.5x^2$
    \begin{itemize}
        \item SymPlex: \texttt{(0.5*(x*(-(x) + (5.0*t)))) + cos((2.0*y))}
        \item SSDE: \texttt{-0.02555*t*x\^{}2 + t*sin(2*y + 1.5725536) + t - 0.02555*x\^{}2 + sin(2*y + 1.5725536)}
        \item FEX: \texttt{(-0.9307*((((1.9975*(sin(t))+0.0184*(sin(x))+-0.0045*(sin(y))+0.2556))\\-((-3.8508*(-(t))+0.6605*(-(x))+-0.0006*(-(y))+-0.2578))))\^{}2+0.8235)}
        \item PINN+DSR: \texttt{-2*x1 + t*(2*x1 - cos(x1))}
        \item KAN: \texttt{0.1721*(0.005*log(8.9166 - 1.3873*x) + 2.0578*sin(1.0079*y + 1.4314) - 0.0037*cos(6.5794*t + 4.5099) - 0.2353)**2 - 0.4106*(84.4225*log(0.1299*x + 9.3657) + 10.2754*sin(0.2658*t - 3.3628) - 0.0019*cos(1.0499*y - 6.078) - 191.2278)**2 + 14.4212*cos(0.0007*(17.2938*t + 0.5879)**2 + 6.97e-5*(2.2204*x - 0.4122)**2 + 0.0412*cos(2.0621*y - 6.2932) - 8.12) + 3.3748}
    \end{itemize}
\end{itemize}

\paragraph{Hamilton-Jacobi}
\begin{itemize}
    \item True solution: $u(x,y,t)=\begin{cases} \sqrt{x^2+y^2} - t, & \sqrt{x^2+y^2} \geq t \\ 0, & \sqrt{x^2+y^2} < t \end{cases}$
    \begin{itemize}
        \item SymPlex: \texttt{max[0,(sqrt((((1.0*x)*x) + (y*y))) - ((t))]}
        \item SSDE: \texttt{ -t*abs(-x\^{}2)\^{}0.5 + t + abs(-x\^{}2)\^{}0.5}
        \item FEX: \texttt{(0.6810*abs(max(((-0.0003*(sin(t))+-0.0547*(sin(x))+0.0017*(sin(y))+-0.0054)), ((-1.0607*(abs(t))+0.8034*(abs(x))+0.8259*(abs(y))+0.3004))))+-0.0039)}
        \item PINN+DSR: \texttt{(x2 - cos(1.422708857249582*x1) + 0.9382428805260224)*\\exp(-1.0529386972249486*t)}
        \item KAN: \texttt{0.0146*(1.6309*log(12.8874 - 9.7645*x) - 3.5841*sin(1.792*t + 2.2865) + 27.996*sin(0.4789*y + 1.5966) - 31.7366)**2 + 0.0028*exp(0.0021*(18.6*y - 8.288)**2 + 2.8372*sin(1.6814*t + 1.5954) - 1.4728*sin(3.719*x - 5.26)) + 0.1995*sin(1.7957*exp(0.9773*t) + 1.7014*log(11.2713 - 7.9086*x) + 3.34736*sin(1.0034*y + 1.7932) - 5.5221) + 0.2447}
    \end{itemize}

     \item True solution: $u(x,y,t)=|x|+|y|+t$
    \begin{itemize}
        \item SymPlex: \texttt{(abs(y) + ((-0.0 + abs(x)) - (-0.2436 * (t / 0.2436))))}
        \item SSDE: \texttt{max(1.0 - cos(t), t + max(1.0 - cos(x), x + max(-y + abs(y*cos(y*(y - abs(y)))), y)))}
        \item FEX: \texttt{(-0.6246*(-(((0.7982*(abs(t))+0.7998*(abs(x))+0.8005*(abs(y))+-0.1432))\\-((-0.8029*(abs((t))+-0.8013*(abs(x))+-0.8006*(abs(y))+0.1407))))+0.1773)	}
        \item PINN+DSR: \texttt{-x1*sin(0.2336406634604177*x1 - 1.7473279933305932) + 1.0023923461698605*x2 + t}
        \item KAN: \texttt{1.5615*exp(-0.0591*sin(3.344*t - 6.271) + 0.7764*sin(0.7765*x + 5.6002) - 1.034*sin(0.5987*y + 2.4682)) + 1.0189*sin(0.654*sqrt(9.5133*t + 4.9446) - 0.0936*sin(1.4299*x + 2.2603) + 0.0671*sin(2.0483*y - 7.4278) - 1.8406) - 0.5665*sin(0.0009*(15.18*t - 8.7903)**2 + 0.8967*sqrt(6.0986*x + 5.8906) + 0.8265*sqrt(6.9141*y + 6.314) - 7.7371) + 0.1685}
    \end{itemize}

    \item True solution: $u(x,t) = \begin{cases}
x - \frac{t}{2} & x \leq \frac{t}{2} \\
0 & x > \frac{t}{2}
\end{cases}$
 \begin{itemize}
        \item SymPlex: \texttt{max[0,(((t + t) + x) + (-1.5 * t))]}
        \item SSDE: \texttt{t*(1 - max(cos(1.0*max(max(1.0, t), 1.0*t)), max(-x + abs(1.0*abs(x)), 1.00000001)))}
        \item FEX: \texttt{(0.8446*(-max(((-0.5920*(-(t))+1.1840*(-(x))+-0.0674)), ((0.0000*((t)\^{}2)+-0.0000*((x)\^{}2)+-0.0674))))+-0.0569)}
        \item PINN+DSR: \texttt{x1 + 0.4973913295197155*t - 0.00585512341908971 + 0.003366022247412346*cos(x1**2)/(0.05801743054817531*x1 - 1)**2}
        \item KAN: 
        \texttt{0.0367*(3.6919*sin(1.5374*x - 7.0127) - 1.7691*cos(1.488*t - 7.7062) + 4.0454)**2 - 0.0241*exp(-0.3833*sin(4.6902*t + 5.2557) + 3.1955*cos(1.7174*x - 2.214)) - 0.016} 
    \end{itemize}

    \item $ u(x,t) = |x| + \frac{t}{2}$
     \begin{itemize}
        \item SymPlex: \texttt{abs((0.5 * t)) + abs(x)}
        \item SSDE: \texttt{0.5*t + max(x + 0.49999958, x\^{}2*(0.90578896 - x)/(x\^{}3 + cos((1.1099373*x - 0.075162254)/x))) - 0.5}
        \item FEX: \texttt{(0.5909*(-(((-0.4223*(t)+-0.8425*(x)+0.1570))+((-0.4238*(abs(t))\\+-0.8498*(abs(x))+0.1427))))+0.1771)	}
        \item PINN+DSR: \texttt{0.3685911419651037*exp(sin(t)) + 0.3685911419651037\\*log(4.698437728826789*x1**2 + 0.41390626642149014)}
        \item KAN: \texttt{1.5535*sin(0.28*sin(3.3566*x - 1.5729) + 46.5592 - 40.6184*exp(-0.008*t)) - 0.0652*cos(0.0005*(8.5555 - 15.0792*t)**2 - 1.9124*cos(4.8933*x - 0.0014) + 3.4232) + 0.9716}
    \end{itemize}
\end{itemize}

\paragraph{Parametric Advection}
\begin{itemize}
    \item True solution: $u(x,y,t;\kappa)=\max\{1 - |x-\kappa t| - |y-\kappa t|, 0\}$
    \begin{itemize}
        \item SymPlex:  \texttt{max[(1-(abs(((k * t) - y))+(1.0*abs((x - (abs(t)*k))))))]}
        \item SSDE: \texttt{ -(1 + 1/cos(exp(exp(t\^{}2 - t))))*log(cos(2.49999998480632e-9*x))}
        \item FEX: \texttt{(-0.2690*sin((((-0.3752*(cos(t))+-0.1229*(cos(x))+-0.1252*(cos(y))\\+-0.2002*(cos(k))+-0.2516))+((-1.4345*(exp(t))+1.2719*(exp(x))+1.3023*(exp(y))\\+-0.7243*(exp(k))+-0.2560))))+0.2257)	}
        \item PINN+DSR: \texttt{sin(k/(x1**2 + x2*k/t + k))}
        \item KAN: \texttt{0.0359*(0.0035*(5.5666 - 15.1003*t)**2 - 0.5089*sin(2.0533*k - 5.0106) + 1.9593*sin(2.1323*y + 5.1286) - 1.9572*cos(2.1382*x - 5.8693) - 1.3651)**2 + 0.0081*(-0.0112*(13.2166 - 19.913*k)**2 - 0.0063*(18.004 - 19.616*t)**2 - 1.0794*cos(2.816*x + 2.0031) - 1.093*cos(2.7966*y + 2.0105) + 2.1919)**3 + 0.0281*(37.0255*sqrt(0.8082*t + 7.9752) + 2.9374*cos(2.7307*k - 2.7356) - 1.2098*cos(1.8331*x - 3.1545) - 1.2004*cos(1.8374*y - 3.1542) - 105.5795)**2 - 5.7881e-6*exp(0.5272*sin(2.2357*y - 4.3275) + 1.4967*cos(2.4549*k - 2.2367) - 11.8356*cos(0.381*t + 2.4254) + 0.5254*cos(2.2469*x + 6.6636)) - 0.0287}
    \end{itemize}

     \item True solution: $u(x,y,t;\kappa) = 2\sin(x-\kappa t)\sin(y - \kappa t)$
    \begin{itemize}
        \item SymPlex: \texttt{2.0*(sin(((k * t) - x))*sin(((k * t) - (-0.0 + y))))}
        \item SSDE: \texttt{(-t + 0.04676*(t*x + t*cos(x) + x*exp(sin(exp(exp(exp(-exp(t)))))) + \\x + exp(sin(exp(exp(exp(-exp(t))))))*cos(x) + cos(x))*\\
        exp(2.812062879082025*cos(0.233057645364174*y))\\ - exp(sin(exp(exp(exp(-exp(t)))))) - 1)\\ *exp(-2.812062879082025*cos(0.233057645364174*y))*sin(x)}
        \item FEX: \texttt{(-0.9510*(-(((0.2382*(sin(t))+-1.4255*(sin(x))+-0.0039*(sin(y))\\+0.0583*(sin(k))+-0.0063))*((0.2376*(sin(t))+0.0045*(sin(x))+-1.4162*(sin(y))\\+0.0632*(sin(k))+-0.0249))))+-0.0051)	}
        \item PINN+DSR: \texttt{exp(-t*log(t))*sin(x1)*sin(x2)}
        \item KAN: \texttt{-0.9746*sin(-0.001*exp(3.7939*k) - 36.0416*log(0.2584*x + 8.4656) + 44.753*log(0.0661*y + 2.733) + 0.0048*cos(4.0712*t + 5.1798) + 24.1655) + 0.0736*sin(-2.3913*sin(2.6296*k - 3.8908) + 4.7596*sin(1.112*t - 2.3346) + 2.7388*sin(0.4017*x + 7.9546) + 1.7815*sin(1.1002*y + 1.5346) + 4.4654) + 0.1914*sin(1.1142*sin(4.1014*k + 1.3522) + 4.7857*sin(0.2152*x + 2.3378) + 3.7566*sin(0.2915*y + 2.293) + 0.6573*cos(1.844*t - 5.7138) - 9.0414) - 0.8441*cos(1.2081*sin(0.8658*t - 9.4502) + 7.5212*sin(0.1352*x + 5.7821) - 0.3877*cos(2.648*k - 2.4958) + 8.7439*cos(0.1146*y - 1.9749) - 5.2825) + 0.0042}
    \end{itemize}
\end{itemize}

\paragraph{Parametric Heat}
\begin{itemize}
    \item True solution: $u(x,y,t;\kappa)=\exp(-x)\exp(-y)\exp(2\kappa t)$
    \begin{itemize}
        \item SymPlex:  \texttt{exp((-((k*(-t)))-(x)))* (exp(-(y)) * exp(((1.0*k) * t)))}
        \item SSDE: \texttt{exp(-t*exp(-x - y) - x - y)}
        \item FEX: \texttt{(0.4461*exp((((-0.0074*((t)\^{}2)+-0.4391*((x)\^{}2)+-0.4455*((y)\^{}2)+-\\0.1241*((k)\^{}2)+0.4223))-((0.5872*(sin(t))+0.4642*(sin(x))+0.4686*(sin(y))\\+0.2501*(sin(k))+-0.4364))))+-0.0702)	}
        \item PINN+DSR: \texttt{exp(-x2 + k*(t - sin(x1))*exp(k))}
        \item KAN: \texttt{1.091e-12*(9.3426 - 1.1838*y)**6.0009*(9.5974 - 1.1842*x)**5.9934\\*exp(-0.0016*(4.9682 - 19.5366*t)**2 + 0.3216*cos(2.68*k - 1.1166)) + 5.5791e-49*(7.5669*k + 5.9751)**3.7466*(6.9658*t + 3.2789)**2.8071*exp(64.2979*exp(-0.016*y) + 32.4132*exp(-0.032*x)) - 0.1706*cos(119.1116*exp(0.008*y) - 0.1836*log(14.4347 - 10.975*t) + 0.1195*log(14.901*k + 1.1798) + 3.3667*log(1.9598*x + 6.2394) - 132.9443) + 0.1175*cos(-0.4637*log(2.4778*t + 0.417) + 0.13*sin(2.3288*y - 7.5748) - 0.1632*cos(2.8358*k + 3.7684) + 0.1842*cos(2.2986*x - 8.9304) + 1.0616) + 0.1117}
    \end{itemize}

     \item True solution: $u(x,y,t;\kappa) = \sin(x)\cos(y) e^{-2\kappa t}$
    \begin{itemize}
        \item SymPlex: \texttt{(-1.0*exp((k*t)))*(sin((-x))*cos(y))}
        \item SSDE: \texttt{-2.629e-7 - 1.0*exp(-21.239998988740282*cos(0.0005091568214335449*y))\\*sin(log(exp(x)))}
        \item FEX: \texttt{(0.6203*sin((((-0.0033*(cos(t))+-0.0144*(cos(x))+0.1902*(cos(y))\\+0.0015*(cos(k))+0.0093))*((-1.1294*((t)\^{}3)+-0.0909*((x)\^{}3)+-0.0003*((y)\^{}3)\\+-0.4277*((k)\^{}3)+3.9474))))+0.0039)	}
        \item PINN+DSR: \texttt{sin(x1)*cos(x2)*cos(t+k)}
        \item KAN: \texttt{0.6537*sin(0.4163*log(15.3317*t + 1.4123) + 0.833*sin(0.878*y - 1.1872) + 0.5411*cos(0.9977*x - 4.7147) - 7.5457 - 1.143*exp(-2.5046*k)) - 1.0089*cos(0.3796*log(14.3003*t + 1.2277) - 0.8264*sin(1.0267*y - 1.6535) + 0.355*cos(0.9995*x - 1.5578) + 0.4552 - 1.071*exp(-2.4462*k)) - 0.6581*cos(-0.4131*log(14.872*t + 1.364) + 0.5304*sin(1.0003*x - 3.1342) - 0.8351*sin(0.8725*y - 1.1702) + 9.1139 + 1.1386*exp(-2.5285*k)) + 1.0283*cos(-0.3737*log(16.4254*t + 1.3994) + 0.3582*sin(0.9968*x + 6.2803) + 0.8163*sin(1.0295*y - 1.6621) + 5.8859 + 1.0541*exp(-2.4509*k)) - 0.024}
    \end{itemize}
\end{itemize}

\subsection{Training Dynamics Across Curriculum Stages}
To verify both the stability of the training process and the effectiveness of the proposed curriculum learning strategy, we visualize the reward evolution over training steps.
Figure~\ref{fig:curriculum_reward} shows the results for the two parametric PDE examples listed in Table~\ref{tab:summary_pdes}.
The figure compares the training dynamics of a single-stage approach, which attempts to solve the full problem at once without curriculum learning, and the proposed curriculum learning strategy.

For curriculum learning, a single SymFormer model is trained sequentially through Stages~1, 2, and~3, and we plot the reward as a function of epochs across all stages.
Note that the reward definition differs across stages, reflecting the increasing task complexity.
In Stage~1, a reward of 1 indicates that the model has correctly learned the spatial representation of the initial condition.
In Stage~2, a reward of 1 corresponds to accurately solving the PDE with the parameter $\kappa$ fixed to 1.
Finally, in Stage~3, a reward of 1 signifies that the model has successfully learned the full parametric PDE solution, including the dependency on $\kappa$, which matches the reward definition used in the single-stage setting.

From the results, we observe that the reward consistently improves within each stage, without severe saturation at specific expressions or noticeable training instability.
In comparison to training without curriculum learning, these results indicate that expressions learned in earlier stages effectively serve as useful priors for subsequent, more complex settings.
Compared to the single-stage approach, which struggles to effectively identify solutions due to the significantly enlarged symbolic search space, the proposed curriculum strategy empirically demonstrates its effectiveness in mitigating search space complexity and facilitating more efficient optimization.
Overall, this empirical evidence supports the stability of our reinforcement learning training scheme and demonstrates the effectiveness of the proposed curriculum in handling increasingly large symbolic search spaces without training collapse.

\begin{figure}[t]
    \centering

    \begin{subfigure}{\linewidth}
        \centering
        \includegraphics[width=0.7\linewidth]{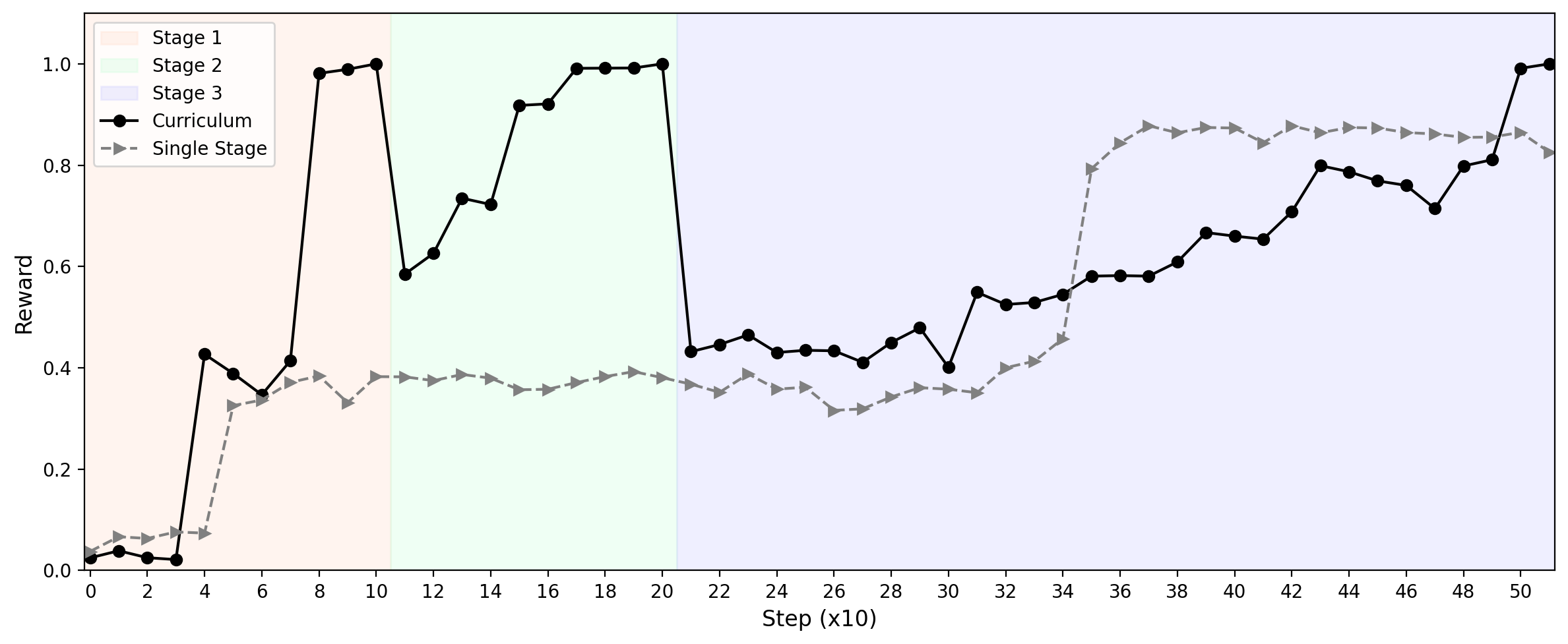}
        \caption{Parametric Advection}
    \end{subfigure}

    \vspace{0.5em}

    \begin{subfigure}{\linewidth}
        \centering
        \includegraphics[width=0.7\linewidth]{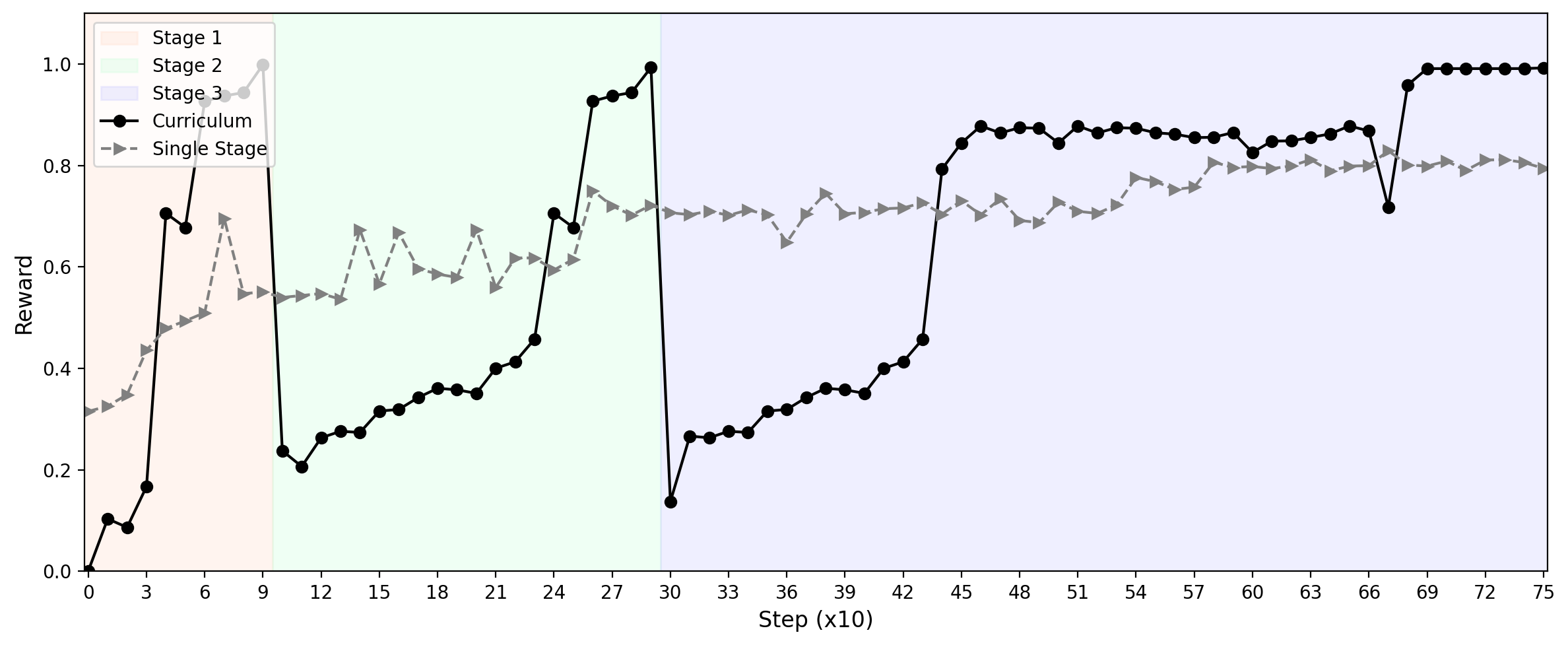}
        \caption{Parametric Heat}
    \end{subfigure}

    \caption{Reward evolution of SymPlex with and without curriculum learning for two parametric PDEs in Table~\ref{tab:summary_pdes}.
The single-stage setting learns the full problem at once without curriculum learning, while curriculum learning trains a single SymFormer sequentially through Stages 1, 2, and 3.
The results show stable optimization across stages under the proposed curriculum.
}
    \label{fig:curriculum_reward}
\end{figure}